%% file: main.tex
\documentclass{article} 
\usepackage{iclr2027_conference,times}

\input{math_commands.tex}

\usepackage{hyperref}
\usepackage{url}
\usepackage{drago}
\usepackage[on]{editing}
\usepackage[T1]{fontenc}

\usepackage{tikz}
\usepackage{tcolorbox}
\usepackage{cleveref}
\usepackage{float}

\iclrfinalcopy

\newcommand{\method}{\textsc{Cfg}}

\newtheorem{example}{Example}

\title{Causally Fair Generation with \\Large Language Models}

\author{
Patrik Okanovi{\'c}$^{1}$ \quad Torsten Hoefler$^{1}$ \quad Drago Ple{\v c}ko$^{2}$ \\
$^{1}$Department of Computer Science, ETH Zurich \\
$^{2}$Department of Statistics \& Data Science, UCLA
}

\begin{document}

\maketitle
\fancyhead{}

\begin{abstract}
\input{sections/0Abstract}
\end{abstract}

\input{sections/1.0Introduction}
\input{sections/1.1RelatedWork}
\input{sections/1.2Preliminaries}
\input{sections/2Method}
\input{sections/3Experiments}

\newpage
\subsection*{AI use statement}

We used AI tools to improve grammar and correct typographical errors in the manuscript, and to assist with code for plotting experimental results.
The original implementation of our method was developed without generative AI assistance. We also used AI tools to help create a subset of the evaluation examples used to test concept detection in our ablation studies. We reviewed and verified all AI-assisted edits, plotting code, and resulting figures for correctness and consistency with the underlying results.

\subsection*{Ethics statement}

This work aims to reduce specified sources of disparity in LLM-generated
content. Our experiments use publicly available survey data and a synthetic
population. Generated stories are intended as fictional examples, and their
statistical patterns should not be interpreted as inherent characteristics
of demographic groups. The choice of causal diagram and business-necessity
exceptions involves substantive judgments about which pathways are
justifiable. These choices require domain expertise and consideration of
affected groups.
The method's fairness conclusions depend on these assumptions and on the
representativeness of the population data. Reducing the evaluated causal
effects does not guarantee that generated content is free of stereotypes
or other harms. Applications in consequential settings therefore require
additional validation and human oversight.

\subsection*{Reproducibility statement}

The source code, computational environment, setup instructions, and
instructions for reproducing all experiments are available in our anonymized
repository: \url{https://anonymous.4open.science/r/cfg/}.
\Cref{app:proofs} provides the theoretical assumptions and proofs.
\Cref{app:census-synth,app:effect-experiments,app:xz-generation} describe
the synthetic data-generating process, data preparation, and experimental
protocols, while
\Cref{sec:concept-agent,sec:drf-vs-xgboost,sec:causal-diagram-retrieval}
detail the component evaluations.

\bibliography{iclr2027_conference}
\bibliographystyle{iclr2027_conference}

\appendix
\newpage
\section*{\centering\Large Supplementary Material:\\ \textit{Causally Fair Generation with Large Language Models}}
\input{sections/4Appendix}

\end{document}

%% file: math_commands.tex
\usepackage{amsmath,amsfonts,bm}

\def\eqref#1{equation~\ref{#1}}

\def\1{\bm{1}}

\DeclareMathAlphabet{\mathsfit}{\encodingdefault}{\sfdefault}{m}{sl}
\SetMathAlphabet{\mathsfit}{bold}{\encodingdefault}{\sfdefault}{bx}{n}

\newcommand{\R}{\mathbb{R}}



%% file: sections/0Abstract.tex
Large language models (LLMs) are increasingly used to generate, complete, and transform information in settings where their outputs can shape consequential decisions, raising concerns about their impact on demographic disparities. In this context, causal inference provides a principled basis for assessing fairness, because it attributes observed disparities to the mechanisms that generated them, which a purely statistical approach cannot do even with infinite data. In LLM generation, a query may request several causally related variables, each of which is both an outcome of interest and a possible cause of other outputs, and the information supplied in the prompt need not follow a topological or a temporal order. This calls for methods that can analyze and selectively remove disparities from such a flexible generation process.
In this paper we introduce Causally Fair Generation with LLMs (\method{}, for short). 
\method{} extracts relevant concepts, grounds generation in a reference population and causal diagram, and removes
user-selected causal effects. \method{} also allows
pathways deemed justifiable for the task's utility to be retained, which is known in legal literature as business necessity. 
Further, we provide formal guarantees for our method when eliminating all discriminatory causal effects in the adapted population model, under appropriate causal assumptions. We evaluate \method{} with four LLMs in three real-world settings based on population data and on a synthetic dataset with a known causal ground truth.

%% file: sections/1.0Introduction.tex
\section{Introduction}

Large language models (LLMs) have become general-purpose tools for producing,
retrieving, and transforming information~\citep{brown2020language}. Their
ability to turn partial instructions into fluent text has led to their use in
education, health care, hiring, and other settings where generated content can
inform consequential decisions
\citep{kasneci2023chatgpt,singhal2023large,an-etal-2024-large}. At scale, an
LLM need not make a final decision itself to have an effect: its summaries,
recommendations, synthetic profiles, and inferred attributes can shape the
information presented to a human or a downstream system. These uses make the
behavior of LLMs across demographic groups an important object of study.
Since these models learn from data shaped by social and historical processes,
their outputs may reproduce existing disparities or introduce new ones.

\input{figures/standard_fairness_model}

Most evaluations of social bias begin with statistical comparisons. They test,
for example, whether prompts referring to different demographic groups produce
different occupations, sentiments, or decision rates
\citep{gallegos-etal-2024-bias}. Such comparisons are a useful starting point
for detecting disparities, but they do not explain how those disparities arose.
A difference may be introduced when the model uses a protected attribute directly, inherited from
another generated attribute, or induced by background characteristics
associated with the protected attribute. The same observed disparity can
therefore reflect different underlying causal mechanisms.

The growing literature on causal fairness \citep{kusner2017counterfactual,kilbertus2017avoiding,chiappa2019path,plecko2022causalfairnessanalysis} offers two important benefits in this context. 
First, it yields human-understandable and interpretable notions of fairness, tied to the causal mechanisms that transmit the disparity between groups. 
Second, the causal approach offers a language aligned with legal notions of discrimination, such as the disparate impact doctrine. 
This includes the notion of business necessity, which specifies which covariates a decision-maker may justifiably use, even if their use results in a disparity between groups. 
In other words, the causal approach allows granular removal or retention of causal effects in line with user requirements, something purely associational methods cannot achieve.

Within causal fairness, it is helpful to distinguish three fairness tasks:
(1) bias detection, which assesses discrimination
in existing outcomes or decisions; (2) fair prediction, which constructs predictors satisfying desired fairness constraints; and (3) fair decision-making, which designs actions to improve real-world outcomes under fairness constraints. In this work, we contribute to a new fairness task, namely that of \emph{fair generation}: constructing outputs that remove selected sources of disparity while retaining pathways deemed permissible.
Fair generation is related to fair prediction, but is markedly more complex. To understand this difference, we compare our setting with the classical machine learning setting, shown in \Cref{fig:causal-diagram}(a), where $X$ denotes the protected attribute, $\mathbf Z$ background covariates, $\mathbf W$ mediators, and $Y$ the
outcome. In classical ML, the task of fair prediction aims to learn a single mechanism for producing $\widehat Y$ from these covariates, with $\widehat Y$ mimicking $Y$ in an appropriate way. In contrast, instead of a single prediction, an LLM can generate multiple causally related attributes from a partially specified profile, implicitly supplying a separate generation mechanism for each desired covariate. This is illustrated in the following example:
\begin{example}[Applicant Profiles] \label{ex:profile}
For safety and privacy reasons, testing of public-sector data-processing systems is increasingly based on synthetic data~\citep{ons2023synthetic}. 
Consider an analyst creating synthetic applicant stories to test an employment-support service, using the American
Community Survey (ACS) as the population data source \citep{uscensus2023acs}. In each prompt, the analyst supplies
sex ($X$) and age ($\mathbf Z$) taken from real data, and requests the model to complete the information on education ($Y_1$), employment status ($Y_2$), and salary ($Y_3$):
\begin{quote}
\emph{Write a story about a 40-year-old person recorded as female, including
their education, employment status, and annual salary.}
\end{quote}
To avoid encoding stereotypical patterns in the resulting outputs, the analyst also wants to remove direct effects of sex on employment and salary, $X\to Y_2$ and $X\to Y_3$, while permitting education-mediated pathways for this scenario, such as $X\to Y_1\to Y_2$ and $X\to Y_1\to Y_3$.
Importantly, education and employment are both generation targets and mediators of the effect of sex on salary, so
fair generation must account for the causal relationships among these variables.
\end{example}
In short, the distribution resulting from the above example can be represented as
\begin{align}
    P_{\mathrm{rw}}(x, \mathbf z)P_{\mathrm{LLM}}(y_1, y_2, y_3 \mid x, \mathbf z),
\end{align}
indicating that $x, \mathbf z$ are drawn from real data, while $y_1, y_2, y_3$ are drawn from the LLM's implied distribution. 
Naturally, the LLM's generative distribution may differ from the real-world one, $P_{\mathrm{LLM}}(y_1, y_2, y_3 \mid x, \mathbf z) \neq P_{\mathrm{rw}}(y_1, y_2, y_3 \mid x, \mathbf z)$. 
Moreover, even if the LLM did generate data according to the real-world distribution, this would simply reproduce existing disparities -- not satisfying the analyst's requirement for removing the direct effects of sex on employment and salary. On the other hand, methods for statistical output filtering \citep{gallegos-etal-2024-bias} equalize group distributions and erase the education-mediated differences the analyst chose to retain. Fair generation, therefore, requires careful modification of the selected mechanisms to ensure generation satisfies the target constraints.

More generally, we arrange the generated attributes $Y_1,\ldots,Y_k$ in a
topological order, with each attribute potentially affecting any later one
(\Cref{fig:causal-diagram}(b)). Interestingly, a prompt may also supply later attributes and request earlier ones, so generation need not follow a topological or a temporal
order (e.g., the analyst may request generation of the sex variable based on age and salary). Fairness constraints must account for these dependencies across the complete generation. In this context, our contributions are:
\begin{enumerate}[label=(\roman*),leftmargin=*,itemsep=2pt,topsep=3pt]
    \item We formulate fair generation under the standard fairness model for topological and non-topological queries (\Cref{subsec:sfm-llms}). We characterize business-necessity removal choices and
    provide formal guarantees for generation under full effect removal (\Cref{thm:joint-debiasing}).

    \item We introduce \method{}, an interactive method for causally fair
    generation with LLMs (\Cref{subsec:pipeline}), which
    extracts concepts from a query, grounds them in population data and a
    causal diagram, removes selected effects while allowing business-necessity
    exceptions, and guides the LLM to express the resulting attributes in text.

    \item We evaluate bias removal with \method{} using four LLMs in three real-world settings. We measure reductions in selected effects at intermediate and final attributes, test non-topological queries, and examine the retention of
    permitted pathways (\Cref{sec:topological-removal-results}).
\end{enumerate}

%% file: figures/standard_fairness_model.tex
\begin{figure}[t]
    \centering
    \begin{tikzpicture}[
        >=latex,
        dag-node/.style={font=\large, inner sep=1.5pt},
        causal/.style={->, thick, shorten <=2pt, shorten >=2pt},
        forward-skip/.style={->, draw=black!70, line width=0.7pt,
            shorten <=2pt, shorten >=2pt,
            preaction={draw=white, line width=2.8pt, arrows=-}},
        confounding/.style={<->, thick, dashed, shorten <=2pt, shorten >=2pt},
        panel/.style={font=\small, align=center}
    ]
        \begin{scope}[xshift=-3.4cm]
            \node[dag-node] (x-a) at (-1.5, 0) {$X$};
            \node[dag-node] (z-a) at (-0.125,1.15) {$Z$};
            \node[dag-node] (w-a) at (-0.125,-1.15) {$W$};
            \node[dag-node] (y-a) at (1.35,-0.65) {$Y$};
            \node[dag-node] (y-h) at (1.35,0.65) {$\hat Y$};

            \draw[confounding, bend left=18] (x-a) to (z-a);
            \draw[causal] (z-a) -- (w-a);
            \draw[causal] (x-a) -- (w-a);
            \draw[causal] (x-a) -- (y-a);
            \draw[causal] (z-a) -- (y-a);
            \draw[causal] (w-a) -- (y-a);
            \draw[causal] (x-a) -- (y-h);
            \draw[causal] (z-a) -- (y-h);
            \draw[causal] (w-a) -- (y-h);

            \node[panel] at (-0.1,-2.1)
                {\textbf{(a)} Standard Fairness Model (SFM)\\ ~\citep{plecko2022causalfairnessanalysis}.};
        \end{scope}

        \begin{scope}[xshift=2.8cm]
            \node[dag-node] (x-b) at (-1.7, 0) {$X$};
            \node[dag-node] (z-b) at (-0.6, 1.15) {$Z$};
            \node[dag-node] (y1-b) at (-0.45,-1.05) {$Y_1$};
            \node[dag-node] (y2-b) at ( 0.85,-1.05) {$Y_2$};
            \node[dag-node] (dots-b) at (2.05,-1.05) {$\cdots$};
            \node[dag-node] (yk-b) at ( 3.25,-1.05) {$Y_k$};

            \draw[confounding, bend left=18] (x-b) to (z-b);
            \draw[causal] (x-b) -- (y1-b.west);
            \draw[causal] (x-b) -- (y2-b.north west);
            \draw[causal] (x-b) -- (yk-b.north west);
            \draw[causal] (z-b) -- (y1-b.north);
            \draw[causal] (z-b) -- (y2-b.north);
            \draw[causal] (z-b) -- (yk-b.north);
            \draw[causal] (y1-b) -- (y2-b.west);
            \draw[causal] (y2-b) -- (dots-b);
            \draw[causal] (dots-b) -- (yk-b.west);
            \draw[forward-skip] ([xshift=3pt]y1-b.north)
                to[out=20,in=160,looseness=0.75]
                ([xshift=3pt]yk-b.north);

            \node[panel] at (0.75,-2.1)
                {\textbf{(b)} Standard Fairness Model for LLMs.};
        \end{scope}
    \end{tikzpicture}
    \caption{Causal diagrams for the standard fairness model and its
    LLM adaptation. Panel (b) orders outputs topologically from left to right.
    Its template includes $Y_i\to Y_j$ for every $i<j$; only selected forward
    edges are drawn for readability. The general formulation permits any
    subset of these edges.}
    \label{fig:causal-diagram}
\end{figure}
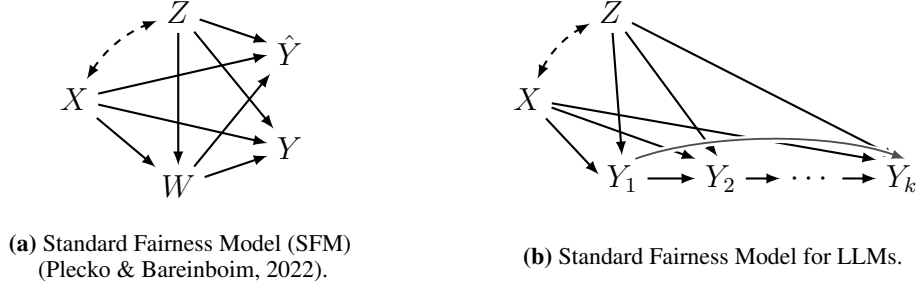

%% file: sections/1.1RelatedWork.tex
\paragraph{Related Work.}
Bias benchmarks measure stereotypes and representational harms
\citep{nangia-etal-2020-crows,nadeem-etal-2021-stereoset,
parrish-etal-2022-bbq,smith-etal-2022-im}; broader audits assess
trustworthiness, implicit bias, and discrimination
\citep{wang2024decodingtrustcomprehensiveassessmenttrustworthiness,
tamkin2023evaluatingmitigatingdiscriminationlanguage,bai2025explicitly}; see
\citet{gallegos-etal-2024-bias} for a survey. These works measure
disparities statistically, whereas \method{} focuses on the selective removal of undesirable causal effects from generated outputs.
The second line of related work elicits causal knowledge from LLMs or evaluates their causal reasoning
\citep{kiciman2024causalreasoninglargelanguage,zevcevic2023causal,
jin2024cladderassessingcausalreasoning,jin2024can,joshi-etal-2024-llms}, or uses them to propose causal graphs \citep{abdulaal2024causal,liu2024discovery};
\method{} uses an LLM to elicit a standard fairness model over the query's concepts, but deploys the diagram for the downstream purpose of bias removal. 
Next, causal approaches to LLM bias use mediation, invariant learning, controlled generation, and prompting
\citep{vig2020investigating,zhou-etal-2023-causal,
madhavan-etal-2023-cfl,li2025prompting}, but a causal model is embedded in the LLM's internal computation, rather than imposed over population variables as in \method{}.
Finally, causal fairness methods address \emph{fair prediction}, constraining a single predicted outcome in the classical ML setting \citep{kusner2017counterfactual,kilbertus2017avoiding,nabi2018fair,
chiappa2019path,plecko2022causalfairnessanalysis,plecko2024fairadapt},
including for images~\citep{joo2020gender,jung2024counterfactually}, facial
expressions~\citep{cheong2023counterfactual}, and
text~\citep{garg2019counterfactual}.
\citet{plecko2026causalbiasdetectiongenerative} addresses \emph{bias detection} in generative AI, focusing on just quantifying causal effects in generated outputs. \method{}, on the other hand, is the first method for \emph{fair generation}: constructing multiple causally related attributes under user-selected fairness constraints, including for non-topological queries, while retaining pathways deemed permissible.

%% file: sections/1.2Preliminaries.tex
\subsection{Preliminaries}

We use structural causal models (SCMs) as our semantic formalism~\citep{pearl_causality}. An SCM is a tuple
$\mathcal{M}=\langle\mathbf{V},\mathbf{U},\mathcal{F},P(\mathbf{U})\rangle$,
where $\mathbf{V}$ and $\mathbf{U}$ are the endogenous (observed) and
exogenous (latent) variables, respectively, and $P(\mathbf{U})$ is a strictly
positive distribution over $\mathbf{U}$. The set $\mathcal{F}$ contains one
structural assignment
$V_i\leftarrow f_i(\operatorname{pa}(V_i),\mathbf{U}_{V_i})$ for each
$V_i\in\mathbf{V}$, where $\operatorname{pa}(V_i)\subseteq\mathbf{V}$ and
$\mathbf{U}_{V_i}\subseteq\mathbf{U}$. The corresponding causal diagram
$\mathcal{G}$ contains a directed edge $V_i\to V_j$ when $V_i$ is an argument
of $f_j$, and a dashed bidirected edge
$V_i\dashleftarrow\mkern-10mu\dashrightarrow V_j$ when their exogenous
causes are dependent. A realization $\mathbf{U}=\mathbf{u}$ is called a
\emph{unit}. The potential outcome $Y_x(\mathbf{u})$ is the value of $Y$ for
unit $\mathbf{u}$ in the submodel $\mathcal{M}_x$, obtained by replacing the
structural assignment for $X$ with $X\leftarrow x$.

We adopt the standard fairness model in \Cref{fig:causal-diagram}(a). Here,
$X\in\{x_0,x_1\}$ is a protected attribute, $\mathbf{Z}$ contains
pre-exposure covariates that may induce spurious association between $X$ and
$Y$, $\mathbf{W}$ contains mediators, and $Y$ is the outcome. We use lowercase
events as shorthand; for example,
$P(y_{x,\mathbf w}\mid x'):=P(Y_{x,\mathbf w}=y\mid X=x')$. The nested potential outcome
$Y_{x,\mathbf{W}_{x'}}$ sets $X=x$ while assigning $\mathbf{W}$ the value it
would attain under $X=x'$. 
Following the causal-fairness literature~\citep{zhang2018fairness,plecko2022causalfairnessanalysis}, the $x$-specific direct, indirect, and spurious effects are $\text{$x$-DE}_{x_0,x_1}(y\mid x):=\ex(Y_{x_1,\mathbf{W}_{x_0}}\mid x)-\ex(Y_{x_0}\mid x)$, $\text{$x$-IE}_{x_1,x_0}(y\mid x):=\ex(Y_{x_1,\mathbf{W}_{x_0}}\mid x)-\ex(Y_{x_1}\mid x)$, and $\text{$x$-SE}_{x_1,x_0}(y):=\ex(Y_{x_1}\mid x_0)-\ex(Y_{x_1}\mid x_1)$. The direct effect changes $X$ from $x_0$ to $x_1$ along the direct path while holding the mediators at their values under $x_0$. The indirect effect measures the reverse change from $x_1$ to $x_0$ along paths through $\mathbf{W}$, with $X$ fixed at $x_1$ along the direct path. The spurious effect compares the two observed groups after fixing $X=x_1$ interventionally. Based on these notions, the total variation of outcome $Y$, defined as $\operatorname{TV}_{x_0,x_1}(y):=\ex(Y\mid x_1)-\ex(Y\mid x_0)$, can be decomposed as \citep{zhang2018fairness}:
\begin{align}
\operatorname{TV}_{x_0,x_1}(y)=\text{$x$-DE}_{x_0,x_1}(y\mid x_0)-\text{$x$-IE}_{x_1,x_0}(y\mid x_0)-\text{$x$-SE}_{x_1,x_0}(y).
\end{align}
In words, total variation decomposes into direct, indirect, and spurious components.

%% file: sections/2Method.tex
\section{\method{}: Causally Fair Generation with LLMs}
We assume the data-generating process is compatible with the SFM in \Cref{fig:causal-diagram}(a), and we additionally assume a Markovian structure within the set of mediators $\mathbf W$, yielding the graph in \Cref{fig:causal-diagram}(b).

\subsection{Topological \& Non-Topological Queries}
\label{subsec:sfm-llms}
\addtocounter{example}{-1}
\begin{example}[Applicant Profiles -- continued]
An analyst supplies sex ($X$) and age ($\mathbf Z$) variables, and uses an LLM to generate applicant stories, while
removing direct effects of sex on employment ($Y_2$) and salary ($Y_3$) and permitting pathways mediated by education ($Y_1$). She samples from
\begin{align}
    P_{\mathrm{rw}}(x, \mathbf z)P_{\mathrm{LLM}}(y_1, y_2, y_3 \mid x, \mathbf z),
\end{align}
with a specific goal for the $P_{\mathrm{LLM}}$ to satisfy a causal fairness constraint described above.
\end{example}
The query in the above example is \emph{topological}, since its supplied variables $X,\mathbf Z$ causally precede the requested variables $Y_1, Y_2, Y_3$. In such a setting, each requested attribute is both an outcome and a potential mediator for later outputs.
Since part of the data comes from the user (real world), and part of the data from the LLM, we need to represent real-world and generative mechanisms within a common SCM. 
Specifically, we introduce one selector indicator for each structural mechanism, $\mathbf S=(S_0,\ldots,S_k)$, where $S_j\in\{\mathrm{rw},\mathrm{gm}\}$. The selector $S_0$ determines how the root variables $(X,\mathbf Z)$ are generated, while $S_i$, $i\geq 1$, determines the mechanism generating $Y_i$. Each selector chooses both the corresponding structural function and its exogenous distribution. Thus, for a fixed selection $\mathbf s=(s_0,\ldots,s_k)$, the induced SCM
$\mathcal M_{\mathbf s}$ is
\begin{align}
    (X,\mathbf Z) &\leftarrow f_0^{s_0}(U_0),
    & U_0 &\sim P_0^{s_0},\\
    Y_i \quad &\leftarrow f_i^{s_i}(\operatorname{pa}_G(Y_i),U_i),
    & U_i &\sim P_i^{s_i}, \qquad i=1,\ldots,k,
\end{align}
where $U_0,\ldots,U_k$ are mutually independent. Hence, different choices of $\mathbf s$ correspond to replacing individual real-world mechanisms by their generative counterparts while leaving the remaining mechanisms unchanged. For the applicant-profile example, $X,\mathbf Z$ come from the real world ($s_0 = \mathrm{rw}$), while $(Y_1, Y_2, Y_3)$ are generated from the model, so that $s_i = \mathrm{gm}$ for $i=1,2,3$.
More generally, other generative combinations are possible by varying $\mathbf s$. In particular, the selection $\mathbf s^{\mathrm{rw}}=(\mathrm{rw},\ldots,\mathrm{rw})$ recovers the real-world population SCM and induces the population distribution
\[
P_{\mathrm{rw}}(x,\mathbf z,\mathbf y)
=
P_{\mathcal M_{\mathbf s^{\mathrm{rw}}}}(x,\mathbf z,\mathbf y),
\]
which simply corresponds to the ground-truth data in ACS.
For a mixed selection $\mathbf s$, and a fixed LLM, prompting, and decoding parameters, we assume that the distribution of extracted generated attributes is induced by the corresponding selected SCM,
\[
P_{\mathrm{LLM}}(\mathbf y\mid x,\mathbf z)
=
P_{\mathcal M_{\mathbf s}}(\mathbf y\mid x,\mathbf z).
\]
With this in mind, we can introduce the concept of \emph{population grounding} -- where the generative distribution of the model, e.g., $P_{\mathrm{LLM}}(\mathbf y\mid x,\mathbf z)$, is steered towards the real world $P_{\mathrm{rw}}(\mathbf y\mid x,\mathbf z)$.
Grounding alone, however, reproduces the population's disparities, and thus we may be interested in first removing undesirable causal effects, as described in the sequel.

\paragraph{Outcome Effects \& Business Necessity.}
For each generated target $Y_i$, we set $\mathbf W^{(i)}=\mathbf Y_{<i}$ and adopt the DE, IE, and SE definitions from the preliminaries. Evaluating these quantities in $\mathcal M_{\mathbf s}$ gives
\begin{align}
\operatorname{TV}_{x_0,x_1}(y_i) = \text{$x$-DE}_{x_0,x_1}(y_i\mid x_0)-\text{$x$-IE}_{x_1,x_0}(y_i\mid x_0)-\text{$x$-SE}_{x_1,x_0}(y_i).
\label{eq:outcome-specific-decomposition}
\end{align}
The direct effect bypasses earlier generated variables, the indirect effect passes through them, and the spurious effect captures association induced by $\mathbf Z$. In our example, education ($Y_1$) has no earlier mediator and hence no indirect effect; education can mediate effects on employment ($Y_2$), while both education and employment can mediate effects on salary ($Y_3$). The spurious component reflects the association between sex and age that subsequently influences $Y$. For each target $Y_i$, the analyst specifies whether the direct, indirect, and spurious components should be removed.

Fully removing an indirect effect, however, may be too restrictive when some mediated pathways are substantively justified for the task. Such exceptions may be captured through the notion of \emph{business necessity}, a common consideration in the legal literature \citep{bernhardt1971giggs}. Hence, for each target $Y_i$, the analyst may specify a BN set $\mathbf B_i$ of generated variables on directed paths from $X$ to $Y_i$, possibly including the target itself. A directed path is permitted whenever it contains at least one variable in $\mathbf B_i$, and is therefore exempt from the corresponding DE or IE removal request. In our example, choosing $\mathbf B_3=\{Y_1\}=\{\text{education}\}$ for salary permits both $X\to Y_1\to Y_3$ and $X\to Y_1\to Y_2\to Y_3$, while not exempting $X\to Y_3$ or $X\to Y_2\to Y_3$, which bypass education. Likewise, $\mathbf B_2=\{Y_1\}$ permits the education-mediated pathway to employment $X \to Y_1 \to Y_2$. BN exceptions concern directed pathways; spurious component removal is specified separately.

The choice of $\mathbf B_i$ also raises an identification question. A sufficient condition for identifying (and thus correctly removing) the permitted path effects is that $\mathbf B_i$ is \emph{backward closed}:
\[
\operatorname{pa}_G(V)\cap\operatorname{De}_G(X)\subseteq\mathbf B_i \qquad \text{for every } V\in\mathbf B_i,
\]
where $\operatorname{pa}_G(V)$ denotes the parents of $V$ and $\operatorname{De}_G(X)$ the strict descendants of $X$. Thus, whenever a selected BN variable has a parent that is itself affected by $X$, that parent must also be included. The education-only set used above satisfies this condition because education has no generated parents. When backward closure does not hold, the relevant path effects may still be identified under stronger assumptions. \Cref{app:bn-semantics} offers a broader discussion on business necessity.

We collect the analyst's target-specific DE, IE, and SE removal choices together with the BN exceptions into a policy $\rho$. The full notation allows these choices to vary across targets, but the interface can be simpler in practice; for example, the analyst may specify a single BN set to be applied wherever relevant, or use the same removal rule across multiple generated attributes.

\paragraph{Generation for topological queries.}
The policy $\rho$ defines the \emph{fair world}: starting from the real-world population distribution, we modify the selected causal pathways according to $\rho$ and the causal diagram $G$, obtaining an adapted distribution $P_{\mathrm{adapt}}^\rho$. Generation then amounts to fitting and sampling from the conditional distribution required by the query. In our running example, sex and age are supplied, while education, employment, and salary are requested:
\begin{align}
P_{\mathrm{rw}}(X,\mathbf Z,\mathbf Y)
\xrightarrow{\ \mathsf{Adapt}(G,\rho)\ }
P_{\mathrm{adapt}}^\rho(X,\mathbf Z,\mathbf Y)
\xrightarrow{\ \mathsf{Fit}\ }
\widehat P_{\mathrm{adapt}}^\rho(\mathbf Y\mid X,\mathbf Z).
\label{eq:topological-retrieval}
\end{align}
The first step moves from the observed population into the fair world specified by $\rho$; the second extracts the conditional needed for the particular query. Because this query is topological, with $X$ and $\mathbf Z$ supplied before their generated descendants, the resulting conditional also admits the usual interventional interpretation in the adapted SCM, i.e., $P_{\mathrm{adapt}}^\rho(\mathbf Y\mid X=x,\mathbf Z=\mathbf z)=P_{\mathrm{adapt}}^\rho(\mathbf Y\mid do(X=x),\mathbf Z=\mathbf z)$ under the SFM assumptions.

\paragraph{Non-topological queries.} The flexible nature of LLMs also allows for non-topological queries:
\addtocounter{example}{-1}
\begin{example}[Applicant Profiles -- continued]
The analyst provides the following query to the model:
\begin{quote}
\emph{A 40-year-old applicant has a bachelor's degree, is employed full-time, and earns \$52,000 annually. Complete the profile by providing the applicant's sex.}
\end{quote}
Here, $\mathbf Z$ and $\mathbf Y$ are supplied while $X$ is requested, to be sampled from $P_{\mathrm{LLM}}(X \mid \mathbf Z, \mathbf Y)$. 
\end{example}
Since the supplied covariates include descendants of $X$, the query is called \emph{non-topological}, and it is a unique feature of generative models.
Interestingly, our previous recipe for fair generation can be adapted appropriately: we first move into the fair world $P_{\mathrm{adapt}}^\rho$ and only then condition on the attributes supplied by the query. Thus,
\begin{align}
P_{\mathrm{rw}}(X,\mathbf Z,\mathbf Y)
\xrightarrow{\ \mathsf{Adapt}(G,\rho)\ }
P_{\mathrm{adapt}}^\rho(X,\mathbf Z,\mathbf Y)
\xrightarrow{\ \mathsf{Fit}\ }
\widehat P_{\mathrm{adapt}}^\rho(X\mid\mathbf Z,\mathbf Y).
\label{eq:non-topological-retrieval}
\end{align}
The conceptual difference is therefore not which causal pathways are removed, but which conditional of the fair-world distribution is sampled. Unlike the topological case above, $P_{\mathrm{adapt}}^\rho(X\mid\mathbf Z,\mathbf Y)$ is purely predictive and has no causal interpretation: conditioning on education, employment, and salary does not treat them as causes of sex.

\subsection{Harness for Causally Fair Generation}
\label{subsec:pipeline}
\begin{figure}[t]
    \centering
    \includegraphics[width=\linewidth]{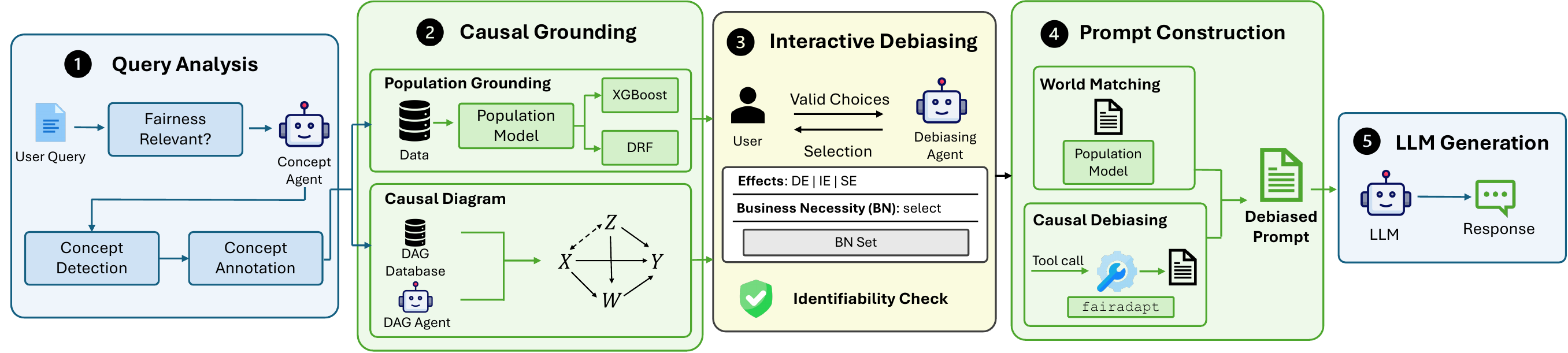}
    \caption{Overview of \method{} for causally fair LLM generation. A
    concept agent extracts and annotates concepts from a fairness-relevant
    query. \method{} then estimates a population model and obtains a causal
    diagram. A debiasing agent presents the user with supported debiasing choices,
    checks their compatibility across outputs, and records the selected effects and
    business-necessity assumptions.
    Finally, \method{} constructs a causally fair output.}
    \label{fig:overview}
\end{figure}

Our next goal is to turn the selected effect removals and BN exceptions into
LLM-generated stories.
\method{} implements this process as an inference-time harness around
the LLM, coordinating population grounding, causal debiasing, and
generation from sampled attribute values.
\Cref{fig:overview} summarizes the five steps of
\method{}: query analysis, population grounding, causal-diagram acquisition,
interactive debiasing, and prompt construction.

\paragraph{(S1) Query analysis.}
A routing classifier first determines whether an input query raises a potential
fairness concern. Selected queries are passed to an LLM-based concept agent,
which uses the task instructions and few-shot examples provided in \Cref{sec:concept-agent}
to identify the variables $V_i\in\mathbf V$ and normalize their explicitly
supplied values. Write $\mathbf C=\mathbf c$ for the variables supplied in context and $\mathbf T$ for those requested for generation. In \Cref{ex:profile},
$\mathbf C=(X,\mathbf Z)$ and $\mathbf T=\mathbf Y$. For the
non-topological query that generates sex from the remaining attributes,
$\mathbf C=(\mathbf Z,\mathbf Y)$ and $\mathbf T=X$. We also support
generation from the full joint distribution, with no supplied attributes:
$\mathbf C=\varnothing$ and $\mathbf T=(X,\mathbf Z,\mathbf Y)$.
For partial completion, the query supplies an initial sequence of outcomes,
$\mathbf C=(X,\mathbf Z,\mathbf Y_{\leq i})$, and requests the remaining
attributes, $\mathbf T=\mathbf Y_{>i}$. In our example, supplying sex, age,
and education leaves employment and salary to be generated.

\paragraph{(S2) Population grounding.}
We estimate $P_{\mathrm{rw}}(\mathbf T\mid\mathbf C=\mathbf c)$ from real-world data of the relevant population, which is assumed to be available. In \Cref{ex:profile}, this is the joint distribution of education, employment, and salary ($\mathbf Y$) in ACS, conditional on the supplied sex and age ($(X,\mathbf Z)$). The fitted population model can be reused across queries. 
Practically, conditional distributions are estimated either via survey-weighted XGBoost classifiers, or distributional random forests (DRF) \citep{cevid2022distributional}.
Joint generation requires composing conditional classifiers in topological order or sampling whole records with forest weights. Details are described in \Cref{sec:drf-vs-xgboost}.

\paragraph{(S3) Causal-diagram acquisition.}
We obtain a causal diagram $G$ over $\mathbf V$ either by matching the
variables to a model in a curated repository or by querying a second LLM-based
agent. The latter returns a candidate SFM assigning the variables their
roles as $X$, $\mathbf Z$, and $\mathbf Y$. In both cases, the diagram is treated as an
explicit structural assumption for the subsequent analysis.
\Cref{sec:causal-diagram-retrieval} evaluates the agent's ability to
propose a causal diagram from the variables in a query.

\paragraph{(S4) Interactive debiasing.}
Our goal is to remove the user-selected direct, indirect, and spurious components while retaining pathways exempted by BN, as specified in $\rho$. We therefore pre-process the data so as to achieve the causal fairness constraint, as described in \Cref{eq:topological-retrieval}. In practice, we use the 
\texttt{fairadapt} method~\citep{plecko2024fairadapt}, which adapts attributes in causal order using conditional quantile models, preserving their conditional ranks as parent values change. 
Further details are provided in \Cref{app:joint-construction,app:bn-semantics}.
In summary, the user records the removal choices in $\rho$, and \method{} returns an
adapted population model:
\begin{align}
  \boxed{
  \mathsf{Adapt}(G,\rho)\colon\ \widehat{P}_{\mathrm{rw}}
  \longmapsto \widehat{P}_{\mathrm{adapt}}^\rho}
  \label{eq:repair-api}
\end{align}
The interface additionally checks that the selected removals admit an identifiable, jointly consistent adaptation across all target variables.
When no effects are selected for removal, generation is calibrated to the real-world reference distribution $\widehat P_{\mathrm{rw}}$, preserving its observed disparities.

\paragraph{(S5) Prompt construction.}
For each query, we jointly sample the requested values
$\mathbf t\sim\widehat P_{\mathrm{adapt}}^\rho(\mathbf T\mid\mathbf C=\mathbf c)$
and insert them into the prompt as fixed constraints alongside the supplied
values $\mathbf c$. The LLM expresses this record in a story. We then verify whether the sampled covariates appear in the story correctly, and if needed, re-prompt the model.

\paragraph{Caching and reuse.}
For a fixed population and set of removal choices, we cache population and adaptation models, causal diagrams, compatibility checks, and retrieval conditionals. Only query analysis, conditioning, record sampling, and LLM generation need be repeated for each request.

\subsection{Theoretical Guarantees}
\label{subsec:guarantees}

Let $\rho^\star$ denote the debiasing policy that removes all directed effects of $X$ on $\mathbf Y$, such as \texttt{fairadapt} used by \method{}. At the population level, this corresponds to recursive substitution (RS):
\begin{align}
Y_i^{\rho^\star}:=f_i^{\mathrm{rw}}\!\left(x_0,\mathbf Z,\mathbf Y_{<i}^{\rho^\star},U_i\right),\qquad i=1,\ldots,k,
\label{eq:recursive-substitution}
\end{align}
so that $\mathbf Y^{\rho^\star}=\mathbf Y_{x_0}$. In practice, \method{} implements this operation using \texttt{fairadapt}.

\begin{theorem}[Debiasing under recursive substitution]
\label{thm:joint-debiasing}
Under $\rho^\star$, the following guarantees hold.

\emph{(i) Topological queries.} For every $\mathbf z$, and every $x$, we have
\begin{align}
P_{\mathrm{adapt}}^{\rho^\star}(\mathbf y\mid x,\mathbf z)
=P(\mathbf Y_{x_0}=\mathbf y\mid\mathbf Z=\mathbf z)
=P_{\mathrm{rw}}(\mathbf y\mid x_0,\mathbf z),
\label{eq:rs-topological}
\end{align}
implying $\mathbf z$-conditional parity.
Further, if the user supplies $(X,\mathbf Z)\sim P^{\mathrm{user}}$ with $P^{\mathrm{user}}(\mathbf z\mid x)=P^{\mathrm{user}}(\mathbf z)$, then the induced joint distribution
$P^{\mathrm{user}}(x,\mathbf z)P_{\mathrm{adapt}}^{\rho^\star}(\mathbf y\mid x,\mathbf z)$ satisfies $X\ci(\mathbf Z,\mathbf Y)$.

\emph{(ii) Non-topological queries.} For every $(\mathbf z,\mathbf y)$,
\begin{align}
P_{\mathrm{adapt}}^{\rho^\star}(x\mid\mathbf z,\mathbf y)
=P_{\mathrm{adapt}}^{\rho^\star}(x\mid\mathbf z)
=P_{\mathrm{rw}}(x\mid\mathbf z).
\label{eq:rs-nontopological}
\end{align}
Thus the generated $X$ is insensitive to the supplied $\mathbf Y$ once $\mathbf Z$ is fixed. If the user supplies $(\mathbf Z,\mathbf Y)\sim P^{\mathrm{user}}$ with
$P^{\mathrm{user}}(\mathbf z,\mathbf y)=P_{\mathrm{adapt}}^{\rho^\star}(\mathbf z,\mathbf y)$, then $P^{\mathrm{user}}(\mathbf z,\mathbf y)P_{\mathrm{adapt}}^{\rho^\star}(x\mid\mathbf z,\mathbf y)
=P_{\mathrm{adapt}}^{\rho^\star}(x,\mathbf z,\mathbf y)$,
so the complete generated population coincides with the debiased joint.
\end{theorem}
The theorem considers two cases.
For topological queries, \Cref{eq:rs-topological} states that after recursive substitution the distribution of the generated record $\mathbf Y$ is the same for every supplied value of $X$, conditional on $\mathbf Z$. The guarantee is therefore $\mathbf z$-conditional. 
Additionally, if the user supplies $(X,\mathbf Z)$ from a distribution in which $X$ and $\mathbf Z$ are independent, this strengthens the conditional guarantee to the independence of $X$ from the complete record $(\mathbf Z,\mathbf Y)$.
For non-topological queries, \Cref{eq:rs-nontopological} gives the corresponding reverse-direction guarantee: for any input $(\mathbf z,\mathbf y)$, the generated $X$ does not depend on $\mathbf y$ once $\mathbf z$ is fixed. Importantly, this pointwise statement does not require the user-supplied $(\mathbf z,\mathbf y)$ to have been drawn from the debiased population. Additionally, however, if the user's inputs $\mathbf Z, \mathbf Y$ follow the debiased marginal $P_{\mathrm{adapt}}^{\rho^\star}(\mathbf z,\mathbf y)$, then the resulting records follow the complete debiased joint distribution $P_{\mathrm{adapt}}^{\rho^\star}(x,\mathbf z,\mathbf y)$.

More generally, the user controls the distribution of the supplied variables $\mathbf C$, while \method{} controls only the conditional distribution of the generated variables $\mathbf T$. Hence the resulting joint distribution is
$P^{\mathrm{user}}(\mathbf C)P_{\mathrm{adapt}}^{\rho^\star}(\mathbf T\mid\mathbf C)$, and coincides with the fully debiased joint only when $P^{\mathrm{user}}(\mathbf C)=P_{\mathrm{adapt}}^{\rho^\star}(\mathbf C)$. This observation also covers hybrid queries such as $\mathbf C=(X,\mathbf Z,\mathbf Y_{\leq i})$: supplied descendants must either already follow the debiased world or be overwritten by their recursively substituted values. We discuss these alternatives in \Cref{app:hybrid-queries}.

%% file: sections/3Experiments.tex
\section{Experiments}

\begin{figure*}[t]
    \centering
    \includegraphics[width=0.95\textwidth]{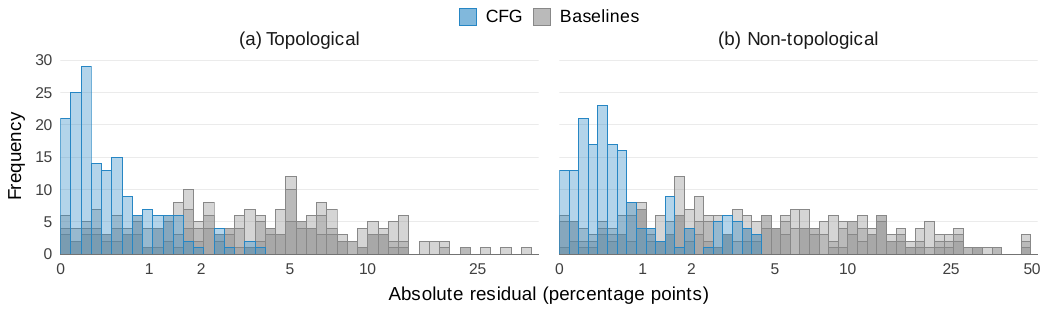}
    \vspace{-0.1in}
    \caption{Absolute recovery residuals under topological and non-topological queries.}
    \vspace{-0.1in}
    \label{fig:recovery-residuals}
\end{figure*}

We evaluate whether \method{} can reproduce the causal effects of a reference population in LLM-generated stories and whether its causal debiasing procedure removes selected sources of disparity from those stories. Both evaluations examine intermediate ($Y_i$, $i<k$) and final outcomes ($Y_k$) to assess causal fairness at these stages of the generated sequence.

\paragraph{Evaluation settings.}
We evaluate generation in three real-world settings: (i) income and substance
use, studying racial disparities using the National Survey on Drug Use and
Health (NSDUH)~\citep{samhsa2023nsduh}; (ii) obesity and diabetes, studying
racial disparities using the Behavioral Risk Factor Surveillance System
(BRFSS)~\citep{cdc2023brfss}; and (iii) employment and salary, studying sex
disparities using the American Community Survey (ACS)~\citep{uscensus2023acs}.
In NSDUH and BRFSS, $X$ is race and $\mathbf Z$ contains age and sex; the
generated attributes $\mathbf Y$ are education, income, and marijuana use for
NSDUH, and education, income, BMI, exercise, and diabetes for BRFSS. In ACS,
$X$ is sex, $\mathbf Z$ contains age, race, and region, and $\mathbf Y$
comprises education, working hours, employer type, and salary.
We also study sex disparities in employment and salary using Census-Synth,
a synthetic population with known ground-truth causal effects, where $X$ is
sex, $\mathbf Z$ is age, and $\mathbf Y$ comprises education, employment, and
salary. \Cref{fig:dataset-fairness-models-v2} of \Cref{app:effect-experiments} presents all the causal diagrams.
\Cref{app:census-synth} provides synthetic data details, while \Cref{app:effect-experiments} provides extended experimental results.

We evaluate four generators: Qwen3-4B~\citep{yang2025qwen3},
Llama~3-8B~\citep{grattafiori2024llama}, Qwen3.8-27B~\citep{qwen38},
and Gemma~4-31B~\citep{gemmateam2026gemma4}. For each model, \method{} is compared against three prompting baselines.
\emph{Vanilla} provides only the task and supplied attributes, with no
external data or retrieval; its generations therefore reflect the
distribution induced by the model's learned weights and the prompt.
\emph{Chain-of-thought (CoT)} adds a request for intermediate reasoning to Vanilla, with the same information available.
\emph{Retrieval-augmented generation (RAG)} supplies the relevant estimated conditional probabilities directly in the prompt. RAG and \method{} use the same fitted population model.
\Cref{app:baselines} details the baseline construction and comparison.

The main topological experiments generate all attributes,
$\mathbf C=\varnothing$ and $\mathbf T=(X,\mathbf Z,\mathbf Y)$;
non-topological queries supply $\mathbf C=(\mathbf Z,\mathbf Y)$
and request $\mathbf T=X$. Completion with supplied $(X,\mathbf Z)$,
as in \Cref{ex:profile}, is evaluated in \Cref{app:xz-generation}.
For all generators and methods, a fixed Qwen3.8-27B annotator extracts
the requested attributes $\mathbf T$ from generated stories as categorical values, which then form the data used to estimate causal effects. \Cref{app:story-annotation} details the annotation procedure.

\subsection{Matching Reference Causal Effects}
We first test whether \method{} can reproduce the reference population's
causal effects before applying any bias removal. We compare TV and its DE, IE, and SE components in generated stories with their reference values for one intermediate and one final target per dataset.
We use a cross-fitted one-step estimator \citep{chernozhukov2018double} for effects in generated stories and survey reference effects; Census-Synth references are computed from the known SCM by Monte Carlo simulation. For each method, we compute absolute errors of the fairness measures relative to the reference values. \Cref{app:effect-experiments} provides outcome definitions, estimation and aggregation details, and mean errors by method.
\Cref{fig:recovery-residuals} pools absolute errors in TV, DE, IE, and SE across generators, datasets, and targets. Under both query types, \method{} matches reference effects better than baselines, with errors concentrating near zero, while baselines show broader distributions with larger errors. Thus, even in the absence of fairness considerations, \method{} can align generation with real-world distributions.

\subsection{Removing Direct, Indirect, and Spurious Effects}
\label{sec:topological-removal-results}

We next evaluate \method{} by removing DE, IE, or SE individually, or all
three together, for the same outcomes. We first use an empty BN set, so no
pathways are exempted from removal. Generated stories are evaluated using
the same annotator and effect estimator as above. We study topological
queries across all four datasets and non-topological queries for Census-Synth
employment. We additionally study Census-Synth salary with a BN set
$\mathbf B=\{\text{education}\}$.

\begin{figure*}[t]
    \centering
    \includegraphics[width=\textwidth]{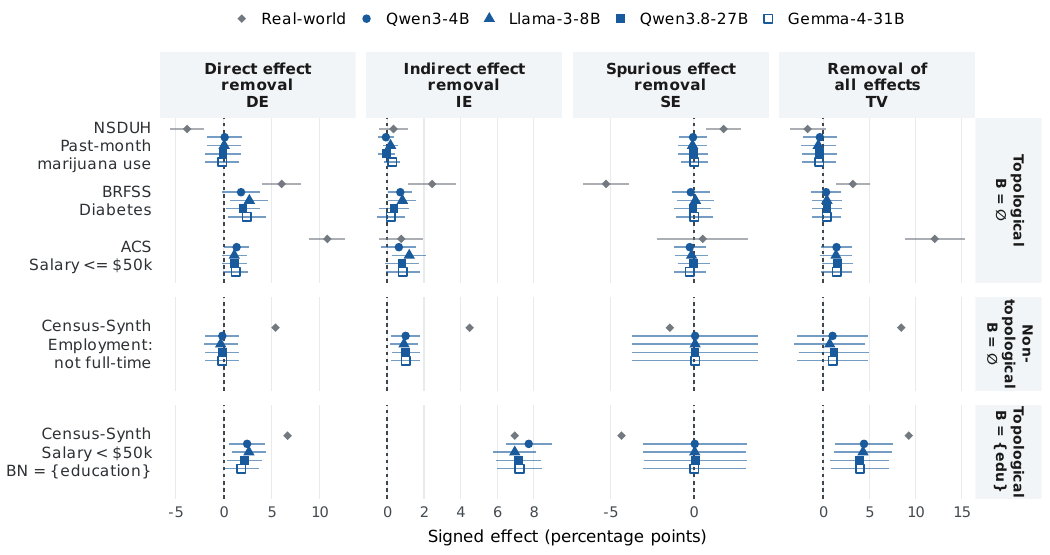}
    \caption{Rows 1--3: topological queries with empty BN set; 
    Row 4: non-topological query; Row 5: topological query with BN requirement. 
    Intervals indicate 95\% confidence intervals.}
    \vspace{-0.1in}
    \label{fig:effect-removal-merged}
\end{figure*}

\paragraph{\method{} aligns fairness behavior across LLM architectures.}
The first four rows of \Cref{fig:effect-removal-merged} show that \method{} brings the selected effects
toward zero for both topological queries and non-topological queries that generate sex from age, education, and employment. The closely aligned estimates across Qwen, Llama, and Gemma show that \method{} induces similar causal fairness behavior, as measured by the targeted effects, across the tested model families and sizes. When all effects are selected for removal, the 95\% confidence intervals for TV include zero in all 20 empty-BN comparisons, corresponding to five dataset-outcome pairs across four generators (see \Cref{app:additional-removal} for details).


\paragraph{Effect removal with a non-empty BN set.}
Returning to the applicant example in \Cref{ex:profile}, the analyst permits education-mediated sex effects on salary. We examine this choice using Census-Synth as a controlled analogue, with $\mathbf B=\{\text{education}\}$. BN permits these pathways and retains larger IE and TV than full effect removal, as shown for topological queries in the bottom row of \Cref{fig:effect-removal-merged}. A remaining salary
gap can therefore reflect the analyst's chosen exceptions. The same pattern holds under non-topological queries; \Cref{app:additional-removal} provides extended results for BN experiments.

\paragraph{Limitations, Strengths, and Future Work.}
An important limitation is that \method{} assumes access to relevant
real-world data through its population-grounding component; such data may be unavailable or insufficiently representative of the target population. Causal effect removal also requires a credible causal diagram, and LLM-proposed diagrams may require domain-expert correction. In addition, \method{} treats the pre-exposure covariates $\mathbf Z$ as a single block, retaining or removing their association with $X$ jointly. \method{}, however, requires no access to the
generator's weights or gradients, making it applicable to black-box LLMs, and supports both topological and non-topological queries.
Future work should address limited population data, sensitivity to causal-diagram uncertainty, and selective adjustment of individual background covariates while accounting for their dependencies.

%% file: sections/4Appendix.tex
The source code for reproducing the results can be found in the anonymized code repository: \url{https://anonymous.4open.science/r/cfg/}.
The repository includes a README file explaining the setup steps. 
All experiments were run on an NVIDIA GH200 GPU, and required 2{,}000 GPU hours.

\input{sections/A1Proofs}

\newpage
\input{sections/A2CensusSynthSCM}

\newpage
\input{sections/A7ExperimentalDetails}

\newpage
\input{sections/A3ConceptAgentEvaluation}

\newpage
\input{sections/A4PopulationFitting}

\newpage
\input{sections/A5CausalDiagramRetrieval}

\newpage
\input{sections/A6IllustrativeConversation}

\newpage
\input{sections/A8XZExperiments}

%% file: sections/A1Proofs.tex
\section{Additional Theoretical Details}
\label[appendix]{app:proofs}

\subsection{Proof of \Cref{thm:joint-debiasing}}
\label{app:joint-construction}

Recall that under $\rho^\star$, recursive substitution fixes the protected attribute input for every downstream mechanism to the value $x_0$:
\begin{align}
Y_i^{\rho^\star}
:=f_i^{\mathrm{rw}}\!\left(x_0,\mathbf Z,\mathbf Y_{<i}^{\rho^\star},U_i\right),
\qquad i=1,\ldots,k.
\label{eq:app-rs}
\end{align}
By recursive substitution, $\mathbf Y^{\rho^\star}=\mathbf Y_{x_0}$.

\begin{proof}[Proof of \Cref{thm:joint-debiasing}]
Because the downstream exogenous variables are independent of the exogenous variables generating $(X,\mathbf Z)$ (according to the SFM in \Cref{fig:causal-diagram}(b)), $\mathbf Y_{x_0}$ is a function only of $\mathbf Z$ and the downstream exogenous variables. Hence, the counterfactual independence $\mathbf Y_{x_0}\ci X\mid\mathbf Z$ holds. Using consistency,
\begin{align}
P_{\mathrm{adapt}}^{\rho^\star}(\mathbf y\mid x,\mathbf z)
&=P(\mathbf Y_{x_0}=\mathbf y\mid X=x,\mathbf Z=\mathbf z) \notag\\
&=P(\mathbf Y_{x_0}=\mathbf y\mid\mathbf Z=\mathbf z)
=P_{\mathrm{rw}}(\mathbf y\mid x_0,\mathbf z),
\end{align}
which proves \Cref{eq:rs-topological}. The debiased joint can therefore be written as
\begin{align}
P_{\mathrm{adapt}}^{\rho^\star}(x,\mathbf z,\mathbf y)
=P_{\mathrm{rw}}(x,\mathbf z)\,
P_{\mathrm{rw}}(\mathbf y\mid x_0,\mathbf z).
\label{eq:app-debiased-joint}
\end{align}
Conditioning this joint in the reverse direction gives
\begin{align}
P_{\mathrm{adapt}}^{\rho^\star}(x\mid\mathbf z,\mathbf y)
&=
\frac{P_{\mathrm{rw}}(x\mid\mathbf z)P_{\mathrm{rw}}(\mathbf y\mid x_0,\mathbf z)}
{\sum_{x'}P_{\mathrm{rw}}(x'\mid\mathbf z)P_{\mathrm{rw}}(\mathbf y\mid x_0,\mathbf z)}
=P_{\mathrm{rw}}(x\mid\mathbf z),
\end{align}
which also equals $P_{\mathrm{adapt}}^{\rho^\star}(x\mid\mathbf z)$ and proves \Cref{eq:rs-nontopological}.

For the topological population statement, if the supplied variables follow $P^{\mathrm{user}}(x,\mathbf z)=P^{\mathrm{user}}(x)P^{\mathrm{user}}(\mathbf z)$, then the resulting joint is
\[
P^{\mathrm{user}}(x)P^{\mathrm{user}}(\mathbf z)P_{\mathrm{adapt}}^{\rho^\star}(\mathbf y\mid\mathbf z),
\]
which factorizes into a distribution of $X$ and a distribution of $(\mathbf Z,\mathbf Y)$, giving $X\ci(\mathbf Z,\mathbf Y)$. For the non-topological population statement, if $P^{\mathrm{user}}(\mathbf z,\mathbf y)=P_{\mathrm{adapt}}^{\rho^\star}(\mathbf z,\mathbf y)$, then
\[
P^{\mathrm{user}}(\mathbf z,\mathbf y)P_{\mathrm{adapt}}^{\rho^\star}(x\mid\mathbf z,\mathbf y)
=P_{\mathrm{adapt}}^{\rho^\star}(x,\mathbf z,\mathbf y)
\]
by the definition of conditional probability.
\end{proof}

\subsection{Business-Necessity Policies}
\label{app:bn-semantics}

\Cref{thm:joint-debiasing} concerns $\rho^\star$, which removes all directed effects of $X$ on the generated outcomes. More general policies may retain selected pathways through business-necessity (BN) exceptions. Identification of such path-specific effects is not always possible. However, backward closure of the BN set gives a sufficient condition for identification \citep[Thm.~6.11]{plecko2022causalfairnessanalysis}.

In practice, we also allow BN specifications not covered by the above identification conditions. In these cases, stronger assumptions are needed. Specifically, we adopt the quantile-preservation assumption used by \texttt{fairadapt}~\citep{plecko2024fairadapt}: conditional latent ranks are preserved when parent values are changed, thereby specifying the joint distribution over counterfactual outcomes. Under this joint coupling, all the resulting quantities are identified.

\subsection{Hybrid Completion Queries}
\label{app:hybrid-queries}
\Cref{thm:joint-debiasing} considers two canonical settings, where either $(X,\mathbf Z)$ are supplied and the entire outcome vector $\mathbf Y$ is generated, or $(\mathbf Z,\mathbf Y)$ are supplied and $X$ is generated. A more general query may instead supply $(X,\mathbf Z,\mathbf Y_{\leq i})$ and request the remaining attributes $\mathbf Y_{>i}$. 
Since the user supplies the values of $X$ and $\mathbf Y_{\leq i}$, the input distribution generated by the user may encode arbitrary relationships of $X \to \mathbf Y_{\leq i}$, which are not under the control of our method. 
If $\mathbf C$ denotes the user-supplied variables and $\mathbf T$ the variables generated by \method{}, the resulting distribution is
\[
P^{\mathrm{user}}(\mathbf C)\,
P_{\mathrm{adapt}}^{\rho^\star}(\mathbf T\mid\mathbf C).
\]
Thus, recovering the complete debiased joint distribution depends on the inputs of the user, and we discuss two possible cases below.

\begin{figure}[t]
\centering
\begin{tikzpicture}[
    >=Latex,
    node distance=1.7cm and 2.2cm,
    every node/.style={font=\Large},
    arr/.style={->, very thick},
    fair/.style={text=orange!90!black}
]
\node (x) {$X$};
\node[below=1.52cm of x] (x0) {$x_0$};
\node[right=of x] (ypre) {$\mathbf Y_{\leq i}$};
\node[right=of ypre] (ypost) {$\mathbf Y_{> i}$};

\node[below=1.4cm of ypre] (ypre0) {$(\mathbf Y_{\leq i})_{x_0}$};
\node[below=1.4cm of ypost] (ypost0) {$(\mathbf Y_{> i})_{x_0}$};

\draw[arr] (x) -- (ypre);
\draw[arr] (ypre) -- (ypost);
\draw[arr] (x0) -- (ypre0);
\draw[arr] (ypre0) -- (ypost0);
\draw[<->, dashed] (ypre) -- (ypre0);
\draw[<->, dashed] (ypost) -- (ypost0);
\end{tikzpicture}
\caption{Twin-network over counterfactual outcomes, related to the hybrid completion setting where the user supplies descendants of $X$. 
The upper path represents the factual values of $X$ and the supplied prefix $\mathbf Y_{\leq i}$, and the generated suffix $\mathbf Y_{>i}$; the lower path represents the recursively substituted world obtained by setting $X=x_0$.}
\label{fig:hybrid-completion}
\end{figure}
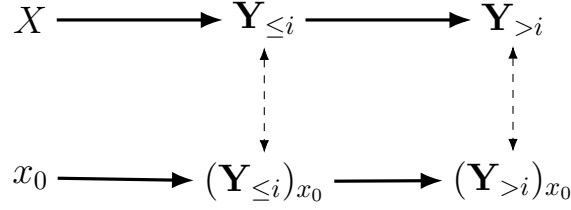

\paragraph{Fair-world user inputs.}
The first possibility is that the user supplies $(X,\mathbf Z,\mathbf Y_{\leq i})$ according to their joint marginal distribution in the debiased world. In particular, the supplied descendants are then realizations of $(\mathbf Y_{\leq i})_{x_0}$, as shown in \Cref{fig:hybrid-completion}.
\method{} then samples the remaining attributes from the corresponding debiased conditional,
\[
P_{\mathrm{adapt}}^{\rho^\star}\!\left(\mathbf Y_{>i}\mid X,\mathbf Z,\mathbf Y_{\leq i}\right).
\]
Under this interpretation, the supplied and generated attributes belong to the same debiased population, and the completed record follows the adapted distribution $P_{\mathrm{adapt}}^{\rho^\star}(x, \mathbf z, \mathbf y)$.

\paragraph{Real-world user inputs.}
Alternatively, the supplied $\mathbf Y_{\leq i}$ may be factual values from the real world, corresponding to the upper path in \Cref{fig:hybrid-completion}. These values already encode part of the factual $X\to\mathbf Y$ relationship. Consequently, simply appending a recursively substituted suffix $(\mathbf Y_{>i})_{x_0}$ combines factual and counterfactual components in the generated record, and should not be interpreted as a draw from the fully debiased population. In this case, the natural way forward is to overwrite the supplied descendants with their recursively substituted counterparts $(\mathbf Y_{\leq i})_{x_0}$ and then generate $(\mathbf Y_{>i})_{x_0}$, thereby moving the entire outcome vector to the lower path in \Cref{fig:hybrid-completion}. In this case, some of the user-provided covariates are overwritten, but the generated records can again be shown to follow the adapted distribution $P_{\mathrm{adapt}}^{\rho^\star}(x, \mathbf z, \mathbf y)$.

%% file: sections/A2CensusSynthSCM.tex
\section{Census-Synth Structural Causal Model}
\label[appendix]{app:census-synth}

We generate $n=8{,}192$ observations from the following ordinal SCM. Let
$U,\epsilon_A,\epsilon_E,\epsilon_M,\epsilon_S
\stackrel{\mathrm{iid}}{\sim}\mathcal{N}(0,1)$, let
$\operatorname{cut}_{\boldsymbol\tau}(r)$ map a latent score $r$ to the
ordered category determined by the cut points $\boldsymbol\tau$, and let
$s(V)$ denote the fixed centered category code of $V$. The structural
assignments are
\begin{equation}
\begin{aligned}
  \text{sex}
  &\leftarrow \operatorname{Bernoulli}
    \!\left(\operatorname{logit}^{-1}(0.9U)\right), \\
  \text{age}
  &\leftarrow \operatorname{cut}_{(-0.8,0,0.8)}
    (0.9U+\epsilon_A), \\
  \text{education}
  &\leftarrow \operatorname{cut}_{(-0.8,0.1,1.0)}
    (-0.15-0.35\,\text{sex}+0.35s(\text{age})+\epsilon_E), \\
  \text{employment}
  &\leftarrow \operatorname{cut}_{(-0.45,0.65)}
    (-0.10-0.25\,\text{sex}-0.10s(\text{age})
      +0.75s(\text{education})+\epsilon_M), \\
  \text{salary}
  &\leftarrow \operatorname{cut}_{(-0.65,0.25,1.15)}
    \!\left(
      \begin{aligned}
        &-0.25-0.35\,\text{sex}+0.25s(\text{age})\\
        &+0.55s(\text{education})+0.75s(\text{employment})+\epsilon_S
      \end{aligned}
    \right).
\end{aligned}
  \label{eq:census-synth-scm}
\end{equation}
Age, education, and salary have four ordered categories, whereas employment
has three. Salary uses the brackets below \$25k, \$25k--\$49,999,
\$50k--\$99,999, and \$100k or more, which also define its four
multiple-choice annotation categories. The shared latent variable $U$ induces
the bidirected $\text{sex}\leftrightarrow\text{age}$ relation; all remaining
parental relations in \Cref{eq:census-synth-scm} are shown in
\Cref{fig:dataset-fairness-models-v2}(d). 
The education term in the salary assignment allows an effect at fixed employment status. Both $\text{sex}\to\text{education}\to\text{salary}$ and $\text{sex}\to\text{education}\to\text{employment}\to\text{salary}$ contribute to the indirect effect of sex on salary; the direct effect of sex concerns the edge $\text{sex}\to\text{salary}$.

\paragraph{Ground-truth effects.}
The structural equations specify how each unit responds to interventions,
allowing us to compute the counterfactual quantities in the effect definitions
directly from the SCM. We evaluate factual outcomes, outcomes under
interventions on sex, and nested counterfactual outcomes in which the earlier
outcomes take their values under the reference sex. These are the
$x$-specific direct, indirect, and spurious effects defined in
\citet{plecko2022causalfairnessanalysis} and used in
\Cref{eq:outcome-specific-decomposition}; we retain the paper's signed
contribution convention for IE and SE.
We approximate the required probabilities using $1{,}000{,}000$ Monte Carlo
draws of the exogenous variables, reusing each draw across the factual and
counterfactual evaluations and conditioning on the factual sex group as
required by the definitions. TV is computed from the factual group
probabilities. These calculations provide a Monte Carlo approximation to the
SCM ground truth for all Census-Synth reference comparisons.

%% file: sections/A7ExperimentalDetails.tex
\section{Effect Recovery and Removal: Additional Details}
\label[appendix]{app:effect-experiments}

\subsection{Targets and Evaluation Protocol}
\label[appendix]{app:effect-protocol}

\Cref{fig:dataset-fairness-models-v2} presents the causal diagrams for all
four datasets.
\input{figures/dataset_fairness_models_v2}

NSDUH, BRFSS, and ACS use publicly available 2023 U.S. survey data.
We sample $n=8{,}192$ records per survey using the survey weights and reuse
the same records across methods and generators. Census-Synth contains
$n=8{,}192$ observations generated from the ordinal SCM over sex, age,
education, employment, and salary specified in \Cref{app:census-synth}.

We evaluate one intermediate and one final target per dataset: income and
monthly marijuana use for NSDUH, obesity and diabetes for BRFSS, employer type
and salary for ACS, and employment and salary for Census-Synth. For target
$Y_i$, the analysis includes $X$, the covariates $\mathbf Z$,
and all preceding outcomes $\mathbf W^{(i)}=\mathbf Y_{<i}$; later outcomes
are excluded. The protected contrast is Black or Hispanic versus White for
NSDUH and BRFSS, and female versus male for ACS and Census-Synth, with the
first group encoded as $x_1$ and the second as $x_0$.

The prespecified outcome events are income below \$50,000 and past-month
marijuana use for NSDUH; obesity (classes 1--3) and diabetes for BRFSS;
salary at most \$50,000 for ACS; and non-full-time employment and salary
below \$50,000 for Census-Synth. The four nominal ACS employer categories
are evaluated as separate one-versus-rest events. This gives seven
intermediate and four final events. All methods use the same source records
and event definitions.

We report DE, IE, and SE as additive contributions to TV, reversing the
signs of the IE and SE terms in \Cref{eq:outcome-specific-decomposition}.
Under this reporting convention, $\mathrm{TV}=\mathrm{DE}+\mathrm{IE}+\mathrm{SE}$;
a positive component contributes toward a higher probability of the outcome
event for $X=x_1$ relative to $X=x_0$.

We use a cross-fitted one-step estimator for effects in annotated generations
for all datasets, including Census-Synth. For NSDUH, BRFSS, and ACS, we also
estimate the reference effects from each survey-weighted sample of $8{,}192$
records. Census-Synth reference effects use the Monte Carlo approximation
to SCM ground truth described in \Cref{app:census-synth}.
All effects are in percentage points. The experiment uses one generation
seed; 95\% Wald intervals quantify estimator uncertainty, not variation
across generation runs. The SCM ground-truth values are treated as fixed
references in the plots.

For removal, we apply each setting separately using
\texttt{fairadapt}~\citep{plecko2024fairadapt}. Its SE-removal implementation
represents the $X$--$\mathbf Z$ association by an $X\to\mathbf Z$ adaptation
edge. Topological generation verbalizes a sampled debiased record. For
non-topological generation, a reverse conditional is fitted to the debiased
joint before sampling $X\mid\mathbf Z,Y_1,\ldots,Y_i$ and verbalizing it;
we evaluate this procedure on Census-Synth employment with an empty BN set
and salary with $\mathbf B=\{\text{education}\}$
(\Cref{app:additional-removal}).

\subsection{Annotating Generated Stories}
\label[appendix]{app:story-annotation}

The same Qwen3.8-27B annotator extracts categorical attributes from generated
stories for all methods, generators, and datasets, in both the matching and
effect-removal experiments.

\paragraph{Extracting categories.}
For each story, the annotator is queried separately for every attribute
requested from the generator. Its prompt contains the story, a question
about the attribute, and a lettered list of the dataset's category labels,
followed by an \emph{Answer not available} option. The model first identifies
the person described, gives a brief explanation based on the text, and then
selects an option. It returns these fields in a structured JSON response.
The instructions specify that, if several people or narratives appear, only
the first is considered. Contradictory or unclear information should lead
to \emph{Answer not available}. BRFSS prompts additionally include fixed,
attribute-specific examples; for education, the instructions select the
highest completed level when several compatible levels are mentioned.

For example, an NSDUH story stating that the person used marijuana during
the past month is assigned the category \emph{yes}. The following prompt
shows the category-selection format, with the generated story inserted at
the indicated placeholder:
\begin{tcolorbox}[colback=gray!10,colframe=gray!50,title=\textbf{Prompt}]
\small\ttfamily\raggedright
Rules:\\
1. If there are multiple people or narratives, focus only on the first one.
Ignore any apparent rewrite, restart, or second attempt later in the story.\\
2. If the same person's information is contradictory, choose Answer not available.\\
3. If the answer is not reasonably clear, choose Answer not available rather than guessing.
\par\medskip
Question: determine the person's marijuana use in the past month.
\par\medskip
A. no\\
B. yes\\
C. Answer not available
\par\medskip
First identify who the story is about, briefly explain your reasoning,
and only then choose the option letter. Respond with one line of JSON
and no other text:
\par\smallskip
\{"target\_person": "<person>", "explanation": "<brief reasoning>", "answer": "<letter>"\}
\par\medskip
<story>\\
{[Generated story]}\\
</story>
\end{tcolorbox}
In this example, the returned letter \texttt{B} is mapped to \emph{yes}.
Analogous questions recover education, income, and the other requested
attributes using their own category lists.

\paragraph{Decoding and incomplete annotations.}
Annotation uses greedy decoding with temperature zero, top-$p$ equal to one,
and a maximum of 300 output tokens per question; the model's thinking mode
is disabled. The returned option letter is parsed and mapped to its category.
Abstentions and responses without a valid option are recorded as missing.
For each query, the generation pipeline provides up to four candidate
stories. We annotate them in order and select the first for which every
requested attribute has a valid category. If none is fully annotated, the
final candidate is retained with its unresolved attributes marked as missing.

\paragraph{From stories to analysis data.}
We save the selected categories together with the story, annotation prompts,
raw responses, explanations, and sample identifier. Each selected story then
contributes one row to the analysis dataset. Supplied attributes retain their
query values, while requested attributes take the values extracted from the
story. Thus, in \method{}, attributes sampled for the generator to verbalize
are recovered from the resulting text for evaluation. Category labels and
their ordering match those in the reference data. The resulting records are
used to estimate TV, DE, IE, and SE following \Cref{app:effect-protocol}.

\paragraph{Agreement with sampled values.}
For \method{}, we compare each attribute value sampled for the generator
with the category subsequently selected by the annotator. We report the
percentage of these pairs that match exactly; missing annotations, including
\emph{Answer not available} (NA), count as nonmatches. Each requested attribute
contributes one comparison, using the selected story after the retry procedure
above. Attributes supplied to the query and not annotated are excluded.

\Cref{tab:cfg-round-trip} reports this agreement for topological recovery on
the same available dataset--outcome combinations for all four generators:
NSDUH income and marijuana use, BRFSS obesity and diabetes, ACS employer type,
and Census-Synth employment. Each generator contributes $49{,}152$ selected
stories and $294{,}912$ sampled--annotated pairs. Agreement ranges from
$98.73\%$ to $99.99\%$.

Disagreements include both NA answers and selections of a category different
from the sampled value; the latter account for most nonmatches in these runs.
An NA answer can arise when the generated story omits an attribute or expresses
it ambiguously, or when the annotator fails to identify information in the text.
A different selected category can similarly reflect a change introduced during
generation or an annotation error. This comparison therefore measures how well
the sampled values are preserved through generation and annotation together.

\input{tables/cfg_sampled_annotated_agreement_table}

\subsection{Baseline Construction}
\label{app:baselines}

For each query, $\mathbf C=\mathbf c$ denotes the supplied attributes and
$\mathbf T$ the requested attributes, as in \Cref{subsec:pipeline}. The
three baselines differ in the information and reasoning instructions given
to the generator; they do not adapt the population distribution to remove
selected causal effects.

\paragraph{Vanilla.}
The model receives the story-generation task and the supplied values
$\mathbf c$, without external data, retrieved statistics, or tool access.
The requested values emerge from the model's learned weights under this prompt.

\paragraph{Chain-of-thought.}
CoT uses the Vanilla setup with an additional instruction to reason through
the task before producing the story. It receives no additional population
information.

\paragraph{Retrieval-augmented generation.}
For each conditioning task, a 500-tree distributional random forest fitted
to the reference records estimates the conditional category probabilities
$\widehat P_{\mathrm{rw}}(V_i=v\mid\mathbf C=\mathbf c)$ for each
requested $V_i\in\mathbf T$. These probabilities are inserted directly
into the generation prompt. For example, when requesting $X=\text{sex}$
given age, education, employment, and salary, the prompt might report
``male: 40\%, female: 60\%'' for that supplied context. These percentages
illustrate the prompt format; the retrieved values depend on the query.
The LLM selects the requested values using this information and writes the
story; no category is fixed by an external draw before generation.

\paragraph{Comparison with \method{}.}
In the recovery experiment, RAG and \method{} share the reference records
and fitted forest. RAG presents conditional marginal probabilities, whereas
\method{} uses the forest's observation weights to sample a complete record
for $\mathbf T$ conditional on $\mathbf C=\mathbf c$, preserving dependencies
among the requested attributes. The sampled values become fixed prompt
constraints. For non-topological queries requesting only $X$, the distinction
is between providing its probabilities and externally sampling its value.
For bias removal, \method{} instead samples from the adapted distribution
$\widehat P_{\mathrm{adapt}}^\rho$. All methods use the same generator
set, annotation procedure, and outcome definitions.

\subsection{Aggregation}
\label[appendix]{app:effect-aggregation}

For the tabulated mean absolute errors, we first
average the four ACS employer events within ACS, then weight the datasets
equally. Removal summaries additionally weight generators equally and compare
absolute effect magnitudes before and after debiasing on matched event--generator
pairs. When all effects are selected for removal, we report residual DE, IE, and SE alongside TV to
assess each component of the remaining disparity.

Aggregate errors are computed over outcome contrasts with valid generations,
annotations, and estimable effects; contrasts for which filtering leaves only
one protected group are excluded. The one-step estimator can produce values
outside the probability scale under weak overlap. Available estimate counts
can therefore differ between methods. The plotting conventions are described
in \Cref{app:additional-recovery}. The large Gemma~4-31B
non-topological RAG error is driven by an unstable NSDUH estimate under weak
overlap and should not be interpreted as a probability-scale disparity.

\subsection{Matching Reference Effects Across Outcomes and Query Orders}
\label[appendix]{app:additional-recovery}

We break down the reference-effect errors in \Cref{fig:recovery-residuals}
by generator, query order, and target position, using the outcomes in
\Cref{app:effect-protocol}.

\paragraph{Recovery throughout the generated sequence.}
Under topological queries, \method{}'s intermediate TV errors range from
$0.40$ to $0.58$ percentage points, and its DE, IE, and SE errors are at
most $0.82$, $0.41$, and $0.27$ points, respectively
(\Cref{tab:recovery-errors}). For example, Llama~3-8B achieves an intermediate
TV error of $0.58$ points, compared with $6.53$ for vanilla prompting and
$6.93$ for RAG.
For final outcomes, mean absolute TV errors range from $0.49$ to $0.66$
points across the four generators. With Llama~3-8B, \method{}'s final-outcome
TV error is $0.63$ points, compared with $4.14$ for Vanilla, $8.92$ for CoT,
and $12.42$ for RAG. Across both target positions and all four generators,
mean absolute errors in DE, IE, and SE are at most $0.82$ points.
Thus, \method{} closely matches both aggregate disparity and its constituent
effects at intermediate and final attributes.

\paragraph{Conditioning on descendants.}
Under non-topological queries, \method{}'s TV errors rise to
$1.58$--$1.68$ points for intermediate outcomes and $1.66$--$1.80$ for
final outcomes. Component errors remain at most $2.18$ for DE, $0.88$ for
IE, and $0.53$ for SE across both target positions. For Llama~3-8B, final
TV error is $1.66$ points with \method{}, compared with $10.27$ for
Vanilla, $7.33$ for CoT, and $10.79$ for RAG. Thus, externally sampling $X$
conditional on the supplied descendants yields more accurate reference-effect
matching than presenting its probabilities to the LLM in this comparison.

\paragraph{Histogram construction.}
Each residual $r$ in \Cref{fig:recovery-residuals} is the absolute difference
between a generated-effect estimate and its reference effect, in percentage
points. Bins have width $0.08$ on the $\log(1+r)$ scale, with axis labels in
the original units and separate horizontal ranges for the two panels.
\method{} is shown in blue; the three baseline histograms are overlaid in
grey under a shared legend entry. Bar heights are unnormalized counts of
event--effect pairs for each method, without the dataset weighting used in
\Cref{app:effect-aggregation}. The $50$-point display cutoff omits $2$ of $580$ topological and
$17$ of $672$ non-topological residuals; all remain in the reported averages
and the cumulative distributions below.

\paragraph{Full residual distributions.}
\Cref{fig:recovery-ecdf} complements the averages in \Cref{tab:recovery-errors}
with empirical cumulative distributions for each method. At a threshold $r$,
the curve gives the proportion of absolute residuals at or below $r$.
Each curve pools effects, generators, datasets, outcome events, and target
positions with equal weight per available residual. All $580$ topological
and $672$ non-topological residuals are included, without the histogram's
display cutoff. The horizontal axis uses $\log(1+r)$ with labels in percentage
points. Under both query orders, \method{}'s curve rises earlier, indicating
that a larger share of its residuals is concentrated near zero.

\begin{figure*}[t]
    \centering
    \includegraphics[width=\textwidth]{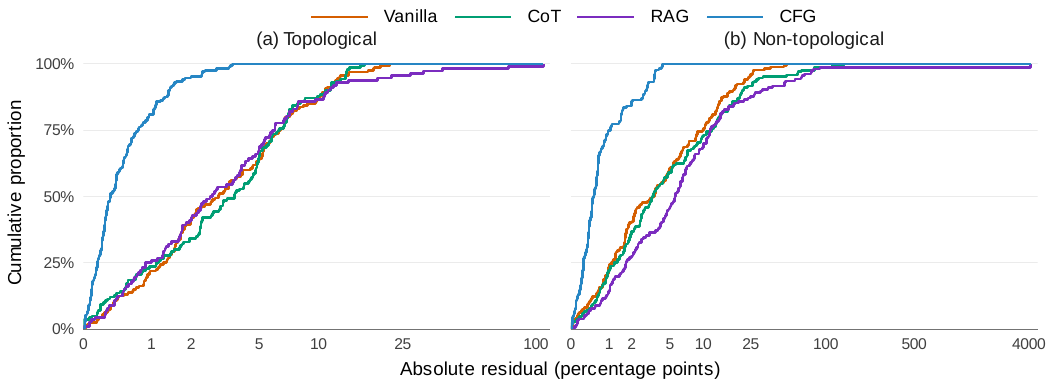}
    \caption{Empirical cumulative distributions of absolute recovery residuals.}
    \label{fig:recovery-ecdf}
\end{figure*}

\begin{table*}[t]
    \centering
    \input{tables/recovery_errors_table}
    \caption{Mean absolute recovery error in percentage points by generator,
    query order, and target position for \method{} and the prompting baselines.
    Dataset contributions are weighted
    equally after averaging ACS employer categories. Bold values indicate
    the smallest error within each comparison.}
    \label{tab:recovery-errors}
\end{table*}

\subsection{Additional Effect-Removal Results}
\label[appendix]{app:additional-removal}

\Cref{fig:removal-empty-bn,fig:removal-census-salary-bn} provide the extended
plots corresponding to \Cref{fig:effect-removal-merged}, including
topological Census-Synth salary with an empty BN set and both query orders
for the non-empty BN comparison.
\Cref{fig:removal-intermediate} complements the main-text empty-BN comparison
with intermediate outcomes under topological queries. The tables report
effect magnitudes before and after debiasing for topological queries
(\Cref{tab:removal-residuals}), non-topological Census-Synth employment
(\Cref{tab:removal-nontopological}), and Census-Synth salary with
$\mathbf B=\{\text{education}\}$ (\Cref{tab:census-salary-bn}).

\begin{figure*}[t]
    \centering
    \includegraphics[width=\textwidth]{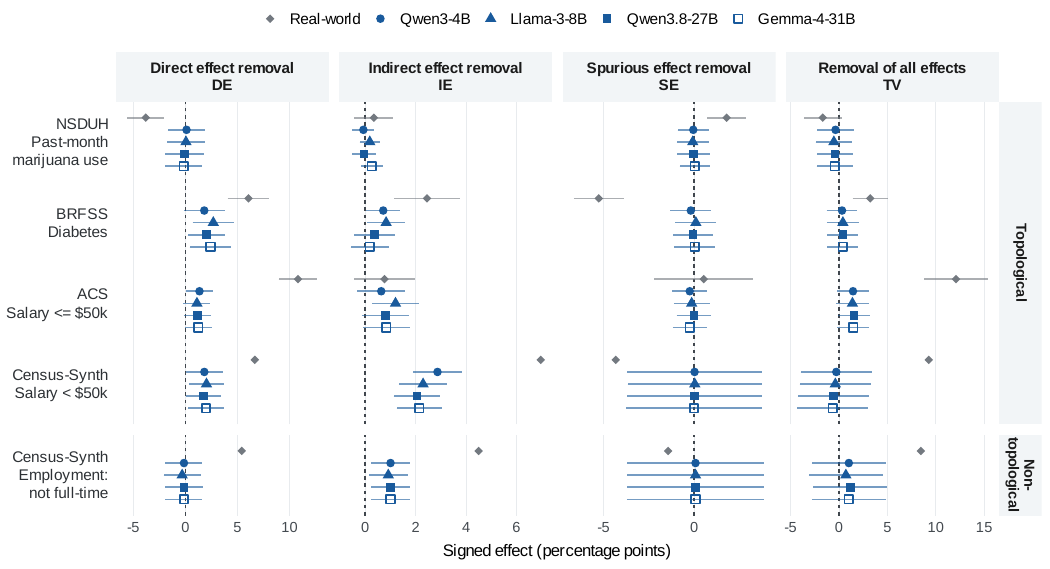}
    \caption{Extended effect-removal results with an empty BN set. The
    upper four rows show final outcomes under topological queries; the
    bottom row shows Census-Synth employment under non-topological queries.
    Columns report the selected effect, or TV when all effects are selected
    for removal. Intervals indicate 95\% confidence intervals.}
    \label{fig:removal-empty-bn}
\end{figure*}

\begin{figure*}[t]
    \centering
    \includegraphics[width=\textwidth]{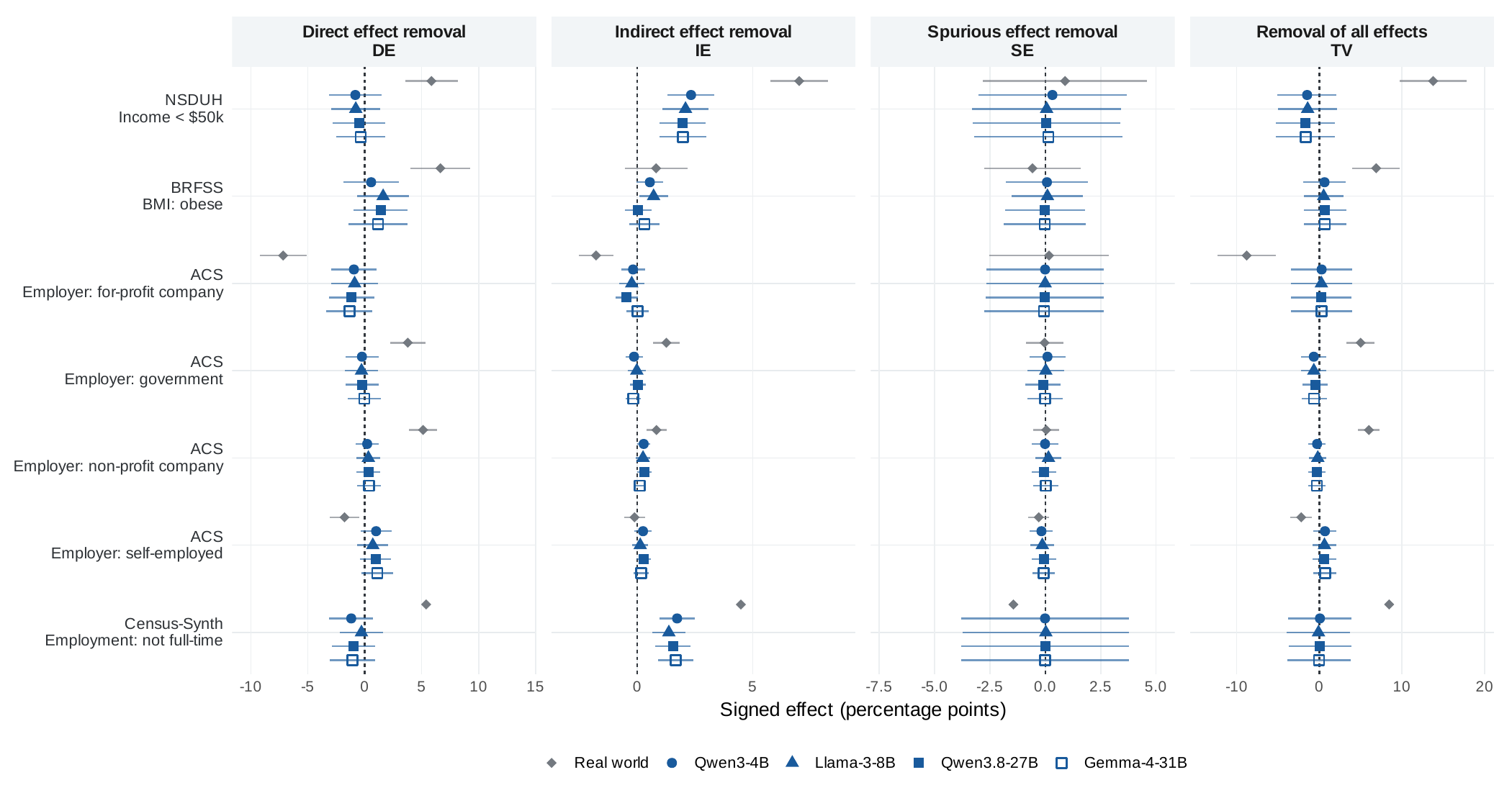}
    \caption{Causal effect removal for intermediate outcomes under
    topological queries. Each panel compares the effect selected for removal
    after debiasing with its original reference value; zero corresponds to complete removal.}
    \label{fig:removal-intermediate}
\end{figure*}

\paragraph{Reduction across targets and generators.}
Across topological intermediate and final outcomes, the selected effect
decreases in magnitude in 169 of 176 comparisons. The remaining seven
involve ACS events whose reference effects are already small, at most
$0.76$ percentage points in magnitude.

\paragraph{Debiasing intermediate attributes.}
Targeting DE reduces its mean absolute magnitude from $5.59$ to $0.75$
percentage points at intermediate outcomes; targeting IE reduces it from
$3.33$ to $1.04$, and targeting SE from $0.76$ to $0.09$. In NSDUH, for
example, the intermediate outcome is income below \$50,000. The reference
direct and education-mediated contributions both increase the probability
of low income for Black or Hispanic individuals, and each is reduced under
its corresponding removal setting (\Cref{fig:removal-intermediate}). This
demonstrates that \method{} can debias an intermediate
attribute that also mediates effects on a later outcome.

\paragraph{Removing all effects.}
\Cref{tab:removal-residuals} shows that removing DE, IE, and SE together
reduces mean absolute TV from $8.65$ to $0.64$ percentage points at
intermediate outcomes. The corresponding DE, IE, and SE residuals are
$0.69$, $0.40$, and $0.09$ points. At final outcomes, TV falls from $6.56$
to $0.54$ points, with component residuals of $0.40$, $0.20$, and $0.07$.
The small aggregate disparity therefore accompanies small direct, indirect,
and spurious contributions, rather than cancellation of large opposing effects.

\begin{table*}[t]
    \centering
    \input{tables/removal_residuals_table}
    \caption{Mean absolute effects in percentage points for the reference and
    \method{}-debiased distributions under topological queries. When all
    effects are selected for removal, we report both TV and its components.}
    \label{tab:removal-residuals}
\end{table*}

\begin{table}[t]
    \centering
    \input{tables/removal_nontopological_table}
    \caption{Mean absolute Census-Synth employment effects in percentage
    points for the reference and \method{}-debiased distributions under non-topological
    queries, averaged across four generators.}
    \label{tab:removal-nontopological}
\end{table}

\paragraph{Business-necessity settings and permitted paths.}
We evaluate salary below \$50,000 using the business-necessity set
$\mathbf B=\{\text{education}\}$ under both query orders, with four
effect-removal choices and four generators. Non-topological queries supply age, education, employment, and
salary and request sex; debiasing precedes conditioning on these supplied values.

\begin{figure*}[t]
    \centering
    \includegraphics[width=\textwidth]{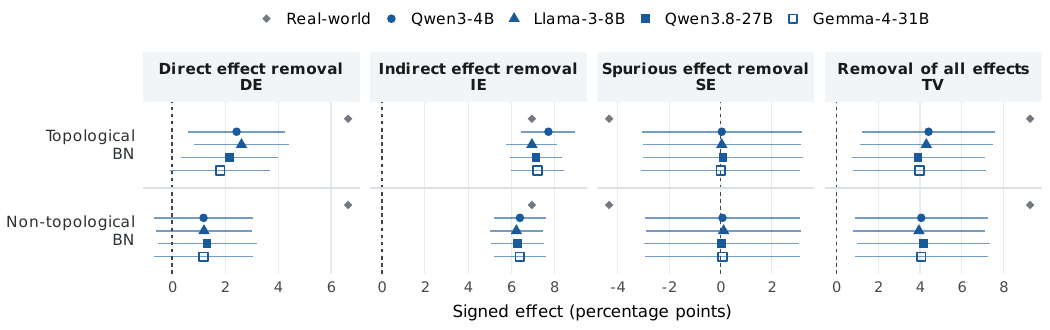}
    \caption{Extended effect-removal results for Census-Synth salary with
    $\mathbf B=\{\text{education}\}$, for both
    topological and non-topological queries. Intervals indicate 95\%
    confidence intervals.}
    \label{fig:removal-census-salary-bn}
\end{figure*}

The BN set permits
$\text{sex}\to\text{education}\to\text{salary}$ and
$\text{sex}\to\text{education}\to\text{employment}\to\text{salary}$.
Paths bypassing education remain subject to removal when their effect family
is selected. These exceptions concern directed pathways; the spurious
component remains subject to removal when selected.

\Cref{fig:removal-census-salary-bn} shows results for both query orders.
The plotted quantities are the full DE, IE, SE, and TV.
A nonzero IE after indirect-effect removal, or TV after removal of all
non-BN effects, can therefore include contributions permitted by BN.
These aggregates do not distinguish permitted contributions from residual
non-BN effects. Zero remains the target for a selected effect with no BN
exception; nonzero estimates cannot then be attributed to an exemption.
\Cref{tab:census-salary-bn} reports the numerical summaries.

\paragraph{Remaining disparity under BN.}
When indirect effects are selected for removal, the mean absolute indirect effect remains close to its
$6.96$-point ground-truth magnitude: $7.26$ points for topological
queries and $6.33$ for non-topological queries
(\Cref{tab:census-salary-bn}). This is consistent with the education-mediated
pathways exempted by BN. When all non-BN effects are selected for removal,
mean absolute TV is $4.15$ points for topological queries and $4.06$ for
non-topological queries, compared with the $9.27$-point reference magnitude.

\begin{table}[tb]
    \centering
    \input{tables/census_salary_bn_table}
    \caption{Mean absolute Census-Synth salary effects in percentage points
    for the reference and \method{}-debiased distributions with
    $\mathbf B=\{\text{education}\}$, averaged across four generators.}
    \label{tab:census-salary-bn}
\end{table}

%% file: figures/dataset_fairness_models_v2.tex
\begin{figure}[t]
    \centering
    \begin{tikzpicture}[
        scale=0.72,
        transform shape,
        >=stealth,
        dag-node/.style={font=\scriptsize, align=center, inner sep=1.5pt},
        confounder-node/.style={draw, rectangle, thin, font=\scriptsize,
            align=center, inner xsep=3pt, inner ysep=2pt},
        causal/.style={->, thin, shorten <=3pt, shorten >=4pt},
        forward/.style={->, thin, shorten <=0.5pt, shorten >=0.5pt,
            preaction={draw=white, line width=2.2pt, arrows=-}},
        chain/.style={forward},
        confounding/.style={<->, thin, dashed,
            shorten <=3pt, shorten >=3pt},
        panel/.style={font=\small, align=center, yshift=-0.4cm}
    ]
        \begin{scope}
            \node[dag-node] (x-a) at (-1.10,0)
                {race};
            \node[confounder-node] (z-a) at (-0.15,1.30)
                {sex \quad age};
            \node[dag-node] (y1-a) at (-0.15,-1.05)
                {edu};
            \node[dag-node] (y2-a) at (1.05,-1.05)
                {income};
            \node[dag-node] (y3-a) at (2.55,-1.05)
                {marijuana use};

            \draw[confounding, bend left=18] (x-a.north) to (z-a.west);
            \draw[causal] (x-a) -- (y1-a.west);
            \draw[causal] (x-a) -- (y2-a.north west);
            \draw[causal] (x-a) -- (y3-a.north west);
            \draw[causal] (z-a) -- (y1-a.north);
            \draw[causal] (z-a) -- (y2-a.north);
            \draw[causal] (z-a) -- (y3-a.north);
            \draw[forward] (y1-a.east) -- (y2-a.west);
            \draw[forward] (y2-a.east) -- (y3-a.west);

            \node[panel] at (0.70,-2.05) {\textbf{(a)} NSDUH.};
        \end{scope}

        \begin{scope}[xshift=4.65cm]
            \node[dag-node] (x-b) at (-1.10,0)
                {race};
            \node[confounder-node] (z-b) at (-0.15,1.30)
                {sex \quad age};
            \node[dag-node] (y1-b) at (-0.15,-1.05)
                {edu};
            \node[dag-node] (y2-b) at (0.95,-1.05)
                {income};
            \node[dag-node] (y3-b) at (2.05,-1.05)
                {BMI};
            \node[dag-node] (y4-b) at (3.20,-1.05)
                {exercise};
            \node[dag-node] (y5-b) at (4.40,-1.05)
                {diabetes};

            \draw[confounding, bend left=18] (x-b.north) to (z-b.west);
            \draw[causal] (x-b) -- (y1-b.west);
            \draw[causal] (x-b) -- (y2-b.north west);
            \draw[causal] (x-b) -- (y3-b.north west);
            \draw[causal] (x-b) -- (y4-b.north west);
            \draw[causal] (x-b) -- (y5-b.north west);
            \draw[causal] (z-b) -- (y1-b.north);
            \draw[causal] (z-b) -- (y2-b.north);
            \draw[causal] (z-b) -- (y3-b.north);
            \draw[causal] (z-b) -- (y4-b.north);
            \draw[causal] (z-b) -- (y5-b.north);
            \draw[forward] (y1-b.east) -- (y2-b.west);
            \draw[forward] (y2-b.east) -- (y3-b.west);
            \draw[forward] (y3-b.east) -- (y4-b.west);
            \draw[forward] (y4-b.east) -- (y5-b.west);

            \node[panel] at (1.85,-2.05) {\textbf{(b)} BRFSS.};
        \end{scope}

        \begin{scope}[xshift=10.45cm]
            \node[dag-node] (x-c) at (-1.10,0)
                {sex};
            \node[confounder-node] (z-c) at (0,1.30)
                {age \quad race \quad region};
            \node[dag-node] (y1-c) at (-0.15,-1.05)
                {edu};
            \node[dag-node] (y2-c) at (0.90,-1.05)
                {hours};
            \node[dag-node] (y3-c) at (2.00,-1.05)
                {employer};
            \node[dag-node] (y4-c) at (3.20,-1.05)
                {salary};

            \draw[confounding, bend left=18] (x-c.north) to (z-c.west);
            \draw[causal] (x-c) -- (y1-c.west);
            \draw[causal] (x-c) -- (y2-c.north west);
            \draw[causal] (x-c) -- (y3-c.north west);
            \draw[causal] (x-c) -- (y4-c.north west);
            \draw[causal] (z-c) -- (y1-c.north);
            \draw[causal] (z-c) -- (y2-c.north);
            \draw[causal] (z-c) -- (y3-c.north);
            \draw[causal] (z-c) -- (y4-c.north);
            \draw[forward] (y1-c.east) -- (y2-c.west);
            \draw[forward] (y2-c.east) -- (y3-c.west);
            \draw[forward] (y3-c.east) -- (y4-c.west);

            \node[panel] at (1.05,-2.05) {\textbf{(c)} Census Income.};
        \end{scope}

        \begin{scope}[xshift=14.90cm]
            \node[dag-node] (x-d) at (-1.10,0)
                {sex};
            \node[confounder-node] (z-d) at (-0.15,1.30)
                {age};
            \node[dag-node] (y1-d) at (-0.15,-1.05)
                {edu};
            \node[dag-node] (y2-d) at (1.15,-1.05)
                {employment};
            \node[dag-node] (y3-d) at (2.70,-1.05)
                {salary};

            \draw[confounding, bend left=18] (x-d.north) to (z-d.west);
            \draw[causal] (x-d) -- (y1-d.west);
            \draw[causal] (x-d) -- (y2-d.north west);
            \draw[causal] (x-d) -- (y3-d.north west);
            \draw[causal] (z-d) -- (y1-d.north);
            \draw[causal] (z-d) -- (y2-d.north);
            \draw[causal] (z-d) -- (y3-d.north);
            \draw[chain] (y1-d.east) -- (y2-d.west);
            \draw[chain] (y2-d.east) -- (y3-d.west);
            \draw[causal, bend left=18] (y1-d.north east) to (y3-d.north west);

            \node[panel] at (0.80,-2.05) {\textbf{(d)} Census-Synth};
        \end{scope}
    \end{tikzpicture}
    \caption{Standard fairness models for NSDUH, BRFSS, ACS, and
    Census-Synth. Generated attributes appear in topological order from left
    to right. In panels (a)--(c), nonconsecutive forward edges are omitted for
    clarity; each generated attribute also points to every later attribute.
    Panel (d) is shown in full.}
    \label{fig:dataset-fairness-models-v2}
\end{figure}

%% file: tables/cfg_sampled_annotated_agreement_table.tex
\begin{table}[t]
\centering
\small
\caption{Sampled--annotated agreement for \method{}.}
\label{tab:cfg-round-trip}
\begin{tabular}{lr}
\toprule
Generator & Agreement (\%) \\
\midrule
Qwen3-4B & 98.96 \\
Llama~3-8B & 98.73 \\
Qwen3.8-27B & 99.75 \\
Gemma-4-31B & 99.99 \\
\bottomrule
\end{tabular}
\end{table}

%% file: tables/recovery_errors_table.tex
\begingroup
\scriptsize
\setlength{\tabcolsep}{2.8pt}
\renewcommand{\arraystretch}{1.08}
\resizebox{\textwidth}{!}{%
\begin{tabular}{lllrrrrrrrr}
\toprule
Generator & Order & Method & \multicolumn{4}{c}{Intermediate target} & \multicolumn{4}{c}{Final target} \\
\cmidrule(lr){4-7}\cmidrule(lr){8-11}
 & & & TV & DE & IE & SE & TV & DE & IE & SE \\
\midrule
\multirow{8}{*}{Qwen3-4B} & \multirow{4}{*}{Topo.} & Vanilla & 7.79 & 4.56 & 2.26 & 1.11 & 10.25 & 8.58 & 0.83 & 2.43 \\
 &  & CoT & 9.64 & 5.51 & 4.04 & 0.81 & 7.62 & 9.51 & 0.73 & 2.48 \\
 &  & RAG & 13.92 & 5.88 & 7.16 & 0.88 & 1.74 & 4.80 & 2.30 & 5.36 \\
 &  & \method{} & \textbf{0.41} & \textbf{0.77} & \textbf{0.36} & \textbf{0.20} & \textbf{0.66} & \textbf{0.49} & \textbf{0.40} & \textbf{0.47} \\
\cmidrule(lr){2-11}
 & \multirow{4}{*}{Non-topo.} & Vanilla & 19.90 & 12.38 & 7.85 & 4.99 & 30.68 & 14.59 & 10.04 & 17.58 \\
 &  & CoT & 33.01 & 14.78 & 11.12 & 22.64 & 60.89 & 19.17 & 6.62 & 48.96 \\
 &  & RAG & 13.86 & 11.16 & 8.74 & 3.20 & 6.61 & 6.17 & 9.60 & 9.17 \\
 &  & \method{} & \textbf{1.62} & \textbf{1.76} & \textbf{0.86} & \textbf{0.20} & \textbf{1.66} & \textbf{2.07} & \textbf{0.55} & \textbf{0.48} \\
\midrule
\multirow{8}{*}{Llama-3-8B} & \multirow{4}{*}{Topo.} & Vanilla & 6.53 & 4.43 & 3.01 & 0.84 & 4.14 & 3.14 & 2.40 & 3.10 \\
 &  & CoT & 8.30 & 4.16 & 3.93 & 1.53 & 8.92 & 7.87 & 2.84 & 3.07 \\
 &  & RAG & 6.93 & 7.79 & 1.26 & 1.12 & 12.42 & 10.90 & 2.03 & 7.07 \\
 &  & \method{} & \textbf{0.58} & \textbf{0.82} & \textbf{0.15} & \textbf{0.18} & \textbf{0.63} & \textbf{0.72} & \textbf{0.39} & \textbf{0.52} \\
\cmidrule(lr){2-11}
 & \multirow{4}{*}{Non-topo.} & Vanilla & 8.48 & 6.39 & 4.26 & 1.77 & 10.27 & 7.96 & 4.48 & 1.94 \\
 &  & CoT & 12.48 & 10.00 & 4.61 & 2.95 & 7.33 & 5.29 & 5.03 & 3.77 \\
 &  & RAG & 11.61 & 9.04 & 5.85 & 3.98 & 10.79 & 5.01 & 8.36 & 5.22 \\
 &  & \method{} & \textbf{1.68} & \textbf{1.66} & \textbf{0.85} & \textbf{0.25} & \textbf{1.66} & \textbf{2.18} & \textbf{0.66} & \textbf{0.45} \\
\midrule
\multirow{8}{*}{Qwen3.8-27B} & \multirow{4}{*}{Topo.} & Vanilla & 6.43 & 5.75 & 3.47 & 2.35 & 5.06 & 6.54 & 1.66 & 2.70 \\
 &  & CoT & 6.00 & 5.74 & 2.98 & 1.71 & 2.19 & 4.86 & 1.63 & 4.14 \\
 &  & RAG & 4.30 & 4.07 & 1.88 & 1.85 & 1.83 & 5.49 & 1.48 & 4.55 \\
 &  & \method{} & \textbf{0.41} & \textbf{0.56} & \textbf{0.35} & \textbf{0.27} & \textbf{0.49} & \textbf{0.60} & \textbf{0.31} & \textbf{0.40} \\
\cmidrule(lr){2-11}
 & \multirow{4}{*}{Non-topo.} & Vanilla & 4.04 & 2.03 & 4.11 & 1.41 & 7.95 & 5.55 & 1.54 & 1.64 \\
 &  & CoT & 2.47 & 3.57 & 1.98 & 0.90 & 7.19 & 5.65 & 1.39 & 1.70 \\
 &  & RAG & 4.08 & 2.38 & 2.84 & 0.74 & 6.31 & 2.40 & 4.36 & 1.08 \\
 &  & \method{} & \textbf{1.58} & \textbf{1.68} & \textbf{0.88} & \textbf{0.25} & \textbf{1.80} & \textbf{2.11} & \textbf{0.48} & \textbf{0.53} \\
\midrule
\multirow{8}{*}{Gemma-4-31B} & \multirow{4}{*}{Topo.} & Vanilla & 12.02 & 6.50 & 3.19 & 4.90 & 8.45 & 5.95 & 3.27 & 2.65 \\
 &  & CoT & 8.02 & 5.18 & 3.57 & 1.70 & 8.99 & 4.41 & \textbf{0.20} & 4.47 \\
 &  & RAG & 41.83 & 16.56 & 3.74 & 25.53 & 8.96 & 6.27 & 2.20 & 5.89 \\
 &  & \method{} & \textbf{0.40} & \textbf{0.61} & \textbf{0.41} & \textbf{0.24} & \textbf{0.66} & \textbf{0.78} & 0.42 & \textbf{0.52} \\
\cmidrule(lr){2-11}
 & \multirow{4}{*}{Non-topo.} & Vanilla & 5.34 & 6.33 & 8.40 & 6.37 & 13.63 & 5.63 & 13.95 & 9.13 \\
 &  & CoT & 4.77 & 7.78 & 10.73 & 2.83 & 16.67 & 5.78 & 16.51 & 4.62 \\
 &  & RAG & 1076.38 & 16.83 & 10.68 & 1067.67 & 33.53 & 23.44 & 20.40 & 23.63 \\
 &  & \method{} & \textbf{1.62} & \textbf{1.76} & \textbf{0.86} & \textbf{0.20} & \textbf{1.66} & \textbf{2.07} & \textbf{0.55} & \textbf{0.48} \\
\bottomrule
\end{tabular}%
}
\endgroup

%% file: tables/removal_residuals_table.tex
\begingroup
\small
\setlength{\tabcolsep}{5pt}
\renewcommand{\arraystretch}{1.08}
\begin{tabular}{llrrrr}
\toprule
Removal setting & Effect & \multicolumn{2}{c}{Intermediate target} & \multicolumn{2}{c}{Final target} \\
\cmidrule(lr){3-4}\cmidrule(lr){5-6}
 & & Reference & \method{} & Reference & \method{} \\
\midrule
Direct effect & DE & 5.59 & 0.75 & 6.82 & 1.32 \\
Indirect effect & IE & 3.33 & 1.04 & 2.62 & 0.97 \\
Spurious effect & SE & 0.76 & 0.09 & 2.98 & 0.08 \\
\midrule
\multirow{4}{*}{All effects} & TV & 8.65 & 0.64 & 6.56 & 0.54 \\
 & DE & 5.59 & 0.69 & 6.82 & 0.40 \\
 & IE & 3.33 & 0.40 & 2.62 & 0.20 \\
 & SE & 0.76 & 0.09 & 2.98 & 0.07 \\
\bottomrule
\end{tabular}
\endgroup

%% file: tables/removal_nontopological_table.tex
\begingroup
\small
\setlength{\tabcolsep}{4.5pt}
\renewcommand{\arraystretch}{1.08}
\begin{tabular}{llrr}
\toprule
Removal setting & Effect & Reference & \method{} \\
\midrule
Direct effect  & DE & 5.39 & 0.19 \\
Indirect effect  & IE & 4.49 & 0.98 \\
Spurious effect  & SE & 1.44 & 0.07 \\
All effects & TV & 8.44 & 0.97 \\
\midrule
All effects & DE & 5.39 & 0.26 \\
All effects & IE & 4.49 & 0.65 \\
All effects & SE & 1.44 & 0.05 \\
\bottomrule
\end{tabular}
\endgroup

%% file: tables/census_salary_bn_table.tex
\begingroup
\small
\setlength{\tabcolsep}{5pt}
\renewcommand{\arraystretch}{1.08}
\begin{tabular}{llrrr}
\toprule
Removal setting & Effect & Reference & Topological & Non-topological \\
\midrule
Direct effect & DE & 6.65 & 2.25 & 1.22 \\
Indirect effect & IE & 6.96 & 7.26 & 6.33 \\
Spurious effect & SE & 4.34 & 0.05 & 0.08 \\
All effects & TV & 9.27 & 4.15 & 4.06 \\
\bottomrule
\end{tabular}
\endgroup

%% file: sections/A3ConceptAgentEvaluation.tex
\section{Concept Agent Evaluation}
\label{sec:concept-agent}

The concept agent implements the query-analysis step of \method{} in
\Cref{subsec:pipeline}. Its purpose is to recover the variables represented in
a generation request before population grounding and causal-diagram
acquisition. We evaluate this component independently of the remaining
steps of \method{} so that downstream assumptions do not obscure extraction errors.

\paragraph{Task.}
Given a prompt requesting a person-level generation, the agent returns a set of
concept records. Each record contains (i) a normalized variable name, (ii) the
value supplied in the prompt, if any, (iii) a status indicating whether the
value is \emph{specified} or merely \emph{requested}, and (iv) an exact span of
the prompt supporting the extraction. Requested concepts have a null value and
become generation targets, whereas specified concepts provide conditioning
information. The agent excludes stylistic and formatting instructions, as well
as attributes that the prompt only prohibits the generated text from
mentioning. The input prompt is presented as data rather than as an instruction
to follow, and the agent does not assign the extracted variables to causal
roles or infer edges between them.

\paragraph{Benchmark.}
We construct a diagnostic benchmark of 100 prompts containing 528 gold concept
records. The benchmark comprises 34 in-domain prompts modeled on the health,
substance-use, census, and employment tasks considered in our experiments; 33
out-of-domain prompts spanning, among others, finance, criminal justice,
transportation, housing, and climate risk; and 33 challenge prompts covering
negation, prohibited attributes, repeated requests, prompts with no eligible
concepts, and prompt-injection language. Here, \emph{out-of-domain} refers to
domains outside the applications used elsewhere in this paper, rather than to
a statistically sampled population of user prompts. Every gold evidence span
occurs verbatim in its prompt. Five additional labeled prompts are reserved as
fixed demonstrations and do not occur in the evaluation set.

\paragraph{Models and evaluation protocol.}
We evaluate the same four generators used in the main experiments, each
receiving the same five demonstrations, with no zero-shot comparison.
We use deterministic decoding
($T=0$ and top-$p=1$) and constrain the response to a JSON schema containing
the four fields described above. Each model is evaluated on all 100 prompts
using a maximum context length of 4,096 tokens and at most 512 generated
tokens.

\paragraph{Metrics.}
We report semantically normalized matching throughout. For evaluation, we apply
the same dictionary of unambiguous lexical aliases to every model, such as
\texttt{sex}$\mapsto$\texttt{gender} and
\texttt{age\_group}$\mapsto$\texttt{age} and
\texttt{annual\_salary}$\mapsto$\texttt{salary}, and apply simple
normalization to values. Name precision, recall, and $F_1$ measure recovery of
the resulting concept set. Record-level precision, recall, and $F_1$
additionally require the specified/requested status and normalized value to
match. We also audit schema validity and whether every predicted evidence span
is present verbatim in the input prompt.

\begin{table}[t]
    \centering
    \small
    \setlength{\tabcolsep}{4pt}
    \begin{tabular}{lccccc}
        \hline
        Model & Name $F_1$ & Record $F_1$ & In-domain & Out-of-domain & Challenge \\
        \hline
        Qwen3-4B      & 95.5 & 90.8 & 93.1 & 87.7 & 92.3 \\
        Llama~3-8B    & 91.6 & 85.7 & 90.4 & 83.0 & 78.9 \\
        Qwen3.8-27B   & 97.9 & 94.3 & 96.2 & 91.1 & 97.5 \\
        Gemma~4-31B   & 95.9 & 91.4 & 94.1 & 86.8 & 95.7 \\
        \hline
    \end{tabular}
    \caption{Few-shot performance of the concept agent for \method{} on the 100-prompt benchmark.
    All entries are percentages. Name and record scores are micro-averaged
    after semantic normalization; the final three columns report record $F_1$
    within each benchmark split.}
    \label{tab:concept-agent-results}
\end{table}

\paragraph{Results.}
Qwen3.8-27B performs best across all reported metrics, achieving 97.9\% name
$F_1$ and 94.3\% record $F_1$. Its record $F_1$ exceeds 91\% in each
benchmark split, including the out-of-domain prompts.

\paragraph{Discussion.}
This performance supports using Qwen3.8-27B for concept detection in
\method{}. Since the prompts are curated and the alias dictionary was refined
after inspecting predictions, the results are a component diagnostic rather
than an independent estimate of generalization to arbitrary requests.

%% file: sections/A4PopulationFitting.tex
\section{Population Fitting Comparison}
\label{sec:drf-vs-xgboost}

We compare the XGBoost and distributional random forest (DRF) estimators
described in step (S2) of \Cref{subsec:pipeline}. XGBoost fits categorical
conditional probabilities; DRF uses a maximum mean discrepancy splitting
criterion to detect changes in the response distribution and assign
query-specific weights to observed records~\citep{cevid2022distributional}.

\paragraph{Evaluation protocol.}
We compare the fitted conditional distribution of race given age, sex,
education, income, and monthly marijuana use on the full preprocessed NSDUH
source of 45,133 records.
For each of five seeds, we hold out 20\% of records, stratifying by race and
monthly marijuana use and keeping identical records in the same partition.
The resulting splits each contain 36,105 training and 9,028 held-out records.
Both estimators use the same training and test partition within each split;
XGBoost selects its number of boosting rounds by five-fold cross-validation
on the training partition only, and DRF uses 500 trees. Analogously to the
negative log predictive density used by \citet{cevid2022distributional}, our
primary metric is held-out categorical log loss.
As a complementary measure, we report the multiclass Brier score \citep{gneiting2007strictly}: the sum of
squared differences between the predicted class-probability vector and the
one-hot observed outcome. We average both scores using the survey weights;
lower values indicate better fit. A probability floor of $10^{-6}$ is applied
when computing log loss.

\begin{table}[t]
    \centering
    \small
    \setlength{\tabcolsep}{8pt}
    \begin{tabular}{lcc}
        \hline
        Estimator & Log loss $\downarrow$ & Brier score $\downarrow$ \\
        \hline
        DRF     & $\mathbf{1.1016 \pm 0.0150}$ & $\mathbf{0.5487 \pm 0.0073}$ \\
        XGBoost & $1.1169 \pm 0.0036$ & $0.5519 \pm 0.0070$ \\
        \hline
    \end{tabular}
    \caption{Survey-weighted held-out goodness of fit on NSDUH for the
    population estimators available in \method{}. Entries are
    means $\pm$ standard deviations across five stratified 80/20 splits.
    }
    \label{tab:nsduh-drf-xgboost-fit}
\end{table}

\paragraph{Results.}
DRF achieves slightly lower mean log loss and Brier score than XGBoost under
this evaluation protocol. Together with its joint-record sampling mechanism,
this supports our choice of DRF as the primary population-fitting option.

%% file: sections/A5CausalDiagramRetrieval.tex
\section{LLM Causal Diagram Retrieval}
\label{sec:causal-diagram-retrieval}

\paragraph{Objective and input.}
We evaluate whether an LLM can recover a plausible standard fairness model
(SFM) from the protected attribute $X$ selected by the user and the remaining
variable names extracted from a query. This tests the agent-based
causal-diagram acquisition step (S3) in \Cref{subsec:pipeline}. Here,
\emph{retrieval} means proposing a diagram using background knowledge; the
model does not select an entry from a supplied graph repository. We isolate
this component by supplying the variable names directly, without concept
extraction errors. The input specifies $X$, whose group disparities are of
interest, but neither the baseline set $\mathbf Z$ nor the ordering of
$\mathbf Y=(Y_1,\ldots,Y_k)$. In particular, the final outcome $Y_k$ is not
given: the model must infer the entire order. No observations, case titles,
reference assignments, or target-case edges are supplied.

\paragraph{SFM construction from a topological order.}
The agent recovers the baseline set $\mathbf Z$ and a topological order
$\mathbf Y=(Y_1,\ldots,Y_k)$. Every remaining variable must appear
exactly once in either $\mathbf Z$ or $\mathbf Y$, with $k\geq1$ and an
optionally empty $\mathbf Z$. Baseline variables in $\mathbf Z$ are not
causally affected by $X$ or $\mathbf Y$. We instantiate the SFM in
\Cref{fig:causal-diagram}(b): $X$ and each $Z_j$ have directed edges to every
$Y_i$, we add $Y_i\to Y_j$ for every $i<j$, and a
bidirected edge connects $X$ with each $Z_j$. No other edges are allowed.
Thus, the roles and order determine the entire graph, including direct
connections between nonadjacent outcomes. This template separates role and
ordering errors from unconstrained edge selection; the general formulation
also permits subsets of the forward edges. The diagram is constructed before
debiasing, so potentially discriminatory edges are retained.

\paragraph{Reference graphs.}
We use the NSDUH and ACS diagrams from the paper and ten hand-curated
hypothetical fairness examples. The synthetic examples were developed with
LLM-assisted drafts, manually revised, and reviewed before evaluation; they
specify baseline roles and outcome orders, rather than datasets generated
from simulated observations.
\Cref{tab:diagram-retrieval-cases} lists all twelve reference assignments.
Under the fixed order-based template above, each row specifies the entire
graph. The examples cover several protected attributes and include an
empty-$\mathbf Z$ case. Their references are reviewed causal assumptions,
rather than uniquely identifiable ground truth from variable names.
BRFSS is reserved as one fixed worked input/output example in every prompt
and is excluded from evaluation. Census-Synth is not part of this twelve-case
retrieval benchmark.

\begin{table}[t]
    \centering
    \footnotesize
    \setlength{\tabcolsep}{3pt}
    \renewcommand{\arraystretch}{1.15}
    \begin{tabular}{@{}>{\raggedright\arraybackslash}p{0.14\linewidth}
                         >{\raggedright\arraybackslash}p{0.13\linewidth}
                         >{\raggedright\arraybackslash}p{0.19\linewidth}
                         >{\raggedright\arraybackslash}p{0.49\linewidth}@{}}
        \toprule
        Case & $X$ & $\mathbf Z$ & Reference $\mathbf Y$ order \\
        \midrule
        NSDUH & race & sex, age &
        education $\prec$ income $\prec$ monthly marijuana use \\
        ACS & sex & age group, race, economic region &
        education $\prec$ weekly hours worked $\prec$ employer type $\prec$ annual salary \\
        \midrule
        Hiring & gender & age &
        education $\prec$ relevant experience $\prec$ interview score $\prec$ hiring decision \\
        Promotion & race & age, sex &
        employer training access $\prec$ job performance $\prec$ promotion decision \\
        Salary offer & sex & age &
        relevant experience $\prec$ initial salary offer $\prec$ negotiation response $\prec$ final salary offer \\
        Admissions & first generation & age, sex &
        household resources $\prec$ school preparation $\prec$ entrance exam score $\prec$ admission decision \\
        Scholarship & disability & age &
        academic merit $\prec$ scholarship awareness $\prec$ application completion $\prec$ scholarship award \\
        Lending & ethnicity & age &
        income $\prec$ repayment history $\prec$ credit score $\prec$ loan approval \\
        Rental housing & nationality & age &
        income $\prec$ viewing invitation $\prec$ reference check $\prec$ lease offer \\
        Insurance & age & sex, race &
        driving experience $\prec$ accident history $\prec$ insurance premium \\
        Benefits & language & age, sex &
        digital literacy $\prec$ application completion $\prec$ eligibility assessment $\prec$ benefit receipt \\
        Moderation & language & $\varnothing$ &
        message content $\prec$ writing style $\prec$ automated toxicity score $\prec$ content removal \\
        \bottomrule
    \end{tabular}
    \caption{The twelve graphs used to evaluate causal-diagram acquisition
    for \method{}. The first two are paper references;
    the remaining ten are hand-curated examples. Spaces replace underscores
    in variable identifiers for readability. The symbol $\prec$ denotes
    precedence in the topological order. Adding every forward outcome edge
    to the fixed SFM template fully specifies each reference graph. The target
    input supplies only $X$ and the unordered remaining names; the $\mathbf Z$
    memberships and $\mathbf Y$ orders shown here are withheld.}
    \label{tab:diagram-retrieval-cases}
\end{table}

\paragraph{Models and protocol.}
We evaluate the four open-weight models used in the main experiments:
Qwen3-4B, Llama~3-8B, Qwen3.8-27B, and Gemma~4-31B. They receive identical
instructions requesting baseline roles and the complete outcome order under
the SFM template, together with the same BRFSS demonstration,
with greedy decoding, thinking disabled, an 8,192-token context limit, and
at most 4,096 output tokens. We additionally evaluate ChatGPT~6 Astra as a
frontier-model comparison available to us at the time of access
(22 September 2026), through the agent interface with requested model
identifier \texttt{gpt-6-astra} and high reasoning effort. Each Astra
presentation uses a fresh agent without conversation history. Public web research is permitted,
but none of the sixty agents reported using it. All reported reading only
their assigned prompt. 

For each graph, we use five distinct presentations: the saved variable order
and four permutations generated with seed 42. These presentations are shared
across models, yielding 60 responses per model and 300 responses overall,
while retaining twelve distinct graphs. Every response is generated in a
fresh context and must return $\mathbf Z$, the complete topological order
of $\mathbf Y$, and both edge lists as JSON. In the saved prompts and
responses, adjacent outcome edges encode this order. For the causal diagrams
used here, we expand that encoding by adding every forward outcome edge;
this deterministic construction requires no further LLM inference.
We use no constrained decoding, corrective
retries, or graph correction. The reference graphs, task prompts, and scoring
rules are fixed across all five models, and Astra answers are frozen before
scoring. 

\paragraph{Metrics.}
We report four complementary scores. \emph{Role macro-$F_1$} averages the
$F_1$ scores for membership in $\mathbf Z$ and $\mathbf Y$, excluding the
supplied $X$; a role absent from both reference and prediction receives
$F_1=1$. \emph{Order agreement} is the fraction of pairs in the reference
$\mathbf Y$ sequence that appear in the correct relative order. A pair is
incorrect if either member is assigned to $\mathbf Z$. This gives partial
credit for nearly correct orders without treating role errors as harmless.
\emph{Exact graph accuracy} requires equality of all directed and bidirected
edges in the graphs constructed from the predicted and reference assignments.
For valid outputs, this is equivalent to recovering the baseline set and the
entire topological order exactly. \emph{Valid-output rate} retains the
original response-schema check: a complete, disjoint assignment and all and
only the edges implied by its adjacent-edge encoding and the fixed SFM
connections. We tolerate prose or Markdown
surrounding a single extractable JSON object, without changing its contents.
Structurally invalid outputs receive zero on all four metrics and remain in
the denominator. Completing the forward edges preserves the role and order
scores and exact-match decisions for valid outputs; the original validity
decisions are retained. All four metrics therefore evaluate the same recovered
roles and order under this graph construction. Scores are averaged over
presentations within each graph and then equally over graphs.

\begin{table}[t]
    \centering
    \small
    \setlength{\tabcolsep}{5pt}
    \begin{tabular}{lrrrr}
        \toprule
        Model & Role macro-$F_1$ & Order agreement & Exact graph & Valid output \\
        \midrule
        Qwen3-4B       & 78.1 & 56.4 & 23.3 & 78.3 \\
        Llama~3-8B     & 73.2 & 45.8 & 11.7 & 78.3 \\
        Qwen3.8-27B    & 95.3 & 83.3 & 58.3 & \textbf{100.0} \\
        Gemma~4-31B    & 94.5 & 87.5 & 60.0 & \textbf{100.0} \\
        ChatGPT~6 Astra (high) & \textbf{97.1} & \textbf{95.8} & \textbf{85.0} & \textbf{100.0} \\
        \bottomrule
    \end{tabular}
    \caption{SFM diagram retrieval for \method{} on twelve graphs, with five variable-order
    presentations per graph and model. All entries are percentages, and
    higher is better.}
    \label{tab:diagram-retrieval-results}
\end{table}

\paragraph{Results.}
\Cref{tab:diagram-retrieval-results} shows that Astra achieves the highest
role and order scores and matches 51 of 60 presentations exactly (85.0\%), compared
with 36 of 60 (60.0\%) for Gemma, the strongest open-weight model by exact
accuracy. On the ten synthetic graphs alone, Astra achieves 96.6\% role
macro-$F_1$, 96.7\% order agreement, and 92.0\% exact accuracy, compared with
62.0\% exact accuracy for Gemma. On the two paper graphs, Astra achieves
100.0\% role macro-$F_1$, 91.7\% order agreement, and 50.0\% exact accuracy.
All Astra outputs are valid even under strict JSON parsing. It exactly
matches every presentation of ten graphs; its nine non-exact responses
occur only in ACS and moderation. The distinction between the partial-credit
and exact metrics is consequential: a single reversed pair can invalidate
an exact match while preserving most of the proposed structure.

\paragraph{ACS: working hours and employer type.}
For all five ACS presentations, Astra recovers the reference baseline set
but proposes the order education $\prec$ employer type $\prec$ weekly hours worked
$\prec$ annual salary. The reference instead uses education $\prec$ weekly
hours worked $\prec$ employer type $\prec$ annual salary. Each order induces
all forward outcome edges. The reference can be
interpreted as working-time commitments or constraints shaping the type of
employer a person enters, which then affects salary. Astra's ordering instead
represents employer type determining working-hour requirements, which affect
annual earnings. Both are plausible simplified mechanisms under different
temporal interpretations of the variables. Their names alone do not specify
whether hours precede employer selection or are determined by the current
employer. Astra reverses one of the six reference $\mathbf Y$ pairs, yielding
83.3\% order agreement for ACS but no exact matches. These are disagreements
with a fixed reference, rather than evidence that the alternative ordering
is causally impossible.

\paragraph{Moderation: message content as a mediator or a baseline.}
The reference sets $\mathbf Z=\varnothing$ and orders message content
$\prec$ writing style $\prec$ automated toxicity score $\prec$ content removal.
Here, language denotes the author's primary language. The reference assumes
that language and its associated communication contexts can affect which
content is expressed; content then influences its linguistic realization
and the moderation score. The graph includes direct edges from message
content to the toxicity score and content removal, as well as paths through
writing style.

In three presentations, Astra instead assigns message content to
$\mathbf Z$ and keeps the order writing style $\prec$ automated toxicity score
$\prec$ content removal for $\mathbf Y$. This is plausible when content means
an intended message or externally specified topic that precedes language
realization: changing the language need not change the intended meaning,
while content can directly affect style, the score, and removal. The implied
language $\leftrightarrow$ message content edge allows shared background
causes of language and topic selection. Thus, the reference and this
alternative correspond to different meanings of ``content'': expressed
content shaped by language versus a prior semantic input. Both are
defensible assumptions for suitably defined tasks, and the names supplied
to the model do not resolve that distinction. A fourth response keeps
$\mathbf Z$ empty but places writing style before message content, which
can represent stylistic choices shaping the eventual expression; this is
again more plausible for expressed content than for fixed intended meaning.
The fifth response matches the reference. We retain all four disagreements
as failures under the prespecified exact-match metric.

\paragraph{Implications for the diagram agent.}
This small, hand-curated benchmark supports using Astra to propose an SFM
for user or expert review. The ACS and moderation cases show why the agent
should expose its roles and order: ambiguous variable definitions can yield
different plausible diagrams with different fairness implications. As in
step (S3), the chosen diagram remains a structural assumption for debiasing.

%% file: sections/A6IllustrativeConversation.tex
\section{Illustrative Conversation with \method{}}
\label[appendix]{app:example-conversation}

The following example illustrates a possible conversation between a user and
an LLM using \method{}. It uses the Census-Synth variables and causal diagram
from \Cref{app:census-synth}. The dialogue illustrates the user's choices;
the final story is taken from a saved generation with an empty BN set and
all effects on salary selected for removal. Effect removal concerns the
distribution from which stories are generated; a single story illustrates
the output format.

\begin{tcolorbox}[colback=gray!10,colframe=gray!50,title=\textbf{User: Prompt}]
Generate a short, realistic story about an adult in the Census-Synth
population. Mention their sex, age, education, employment status, and annual
salary. Use sex as the protected attribute.
\end{tcolorbox}

\begin{tcolorbox}[colback=blue!5,colframe=blue!30,title=\textbf{\method{}: Example Output}]
I can use Census-Synth as the reference population. Its causal diagram can
be written as the following relationships:
\begin{quote}
\small\ttfamily\raggedright
sex -> \{education, employment, salary\}\\
age -> \{education, employment, salary\}\\
education -> \{employment, salary\}\\
employment -> salary\\
sex <-> age
\end{quote}
Here, \texttt{->} denotes a direct causal influence on each variable listed
on the right, and \texttt{<->} denotes unobserved common causes of sex and
age. For salary, education and employment
can transmit indirect effects of sex, while the sex--age association can
contribute a spurious effect.

Please confirm whether you want to use this diagram and focus the fairness
analysis on salary. Would you like to remove any effects before I generate
the story? If you choose no removal, I will sample the attributes from the
reference population distribution, aiming to reproduce its relationships
and disparities. For Census-Synth, that reference is the synthetic population.
\end{tcolorbox}

\begin{tcolorbox}[colback=gray!10,colframe=gray!50,title=\textbf{User: Prompt}]
Yes, use this diagram and focus on salary. I want to remove effects before
you generate the story.
\end{tcolorbox}

\begin{tcolorbox}[colback=blue!5,colframe=blue!30,title=\textbf{\method{}: Example Output}]
First, are there any pathways you consider justified for this task and want
to retain? These are called \textbf{business-necessity (BN) exceptions}.
You can select a BN set from education, employment, and salary, or choose none.
A directed pathway from sex to salary is exempt from removal if it contains
at least one variable in your BN set. For example, selecting education
exempts pathways through education. An empty BN set provides no exceptions.
Spurious-effect removal is specified separately from these exceptions.

Which BN set would you like? Also choose which effects on salary
to remove:
\begin{itemize}[leftmargin=*, nosep]
  \item \textbf{Direct effect:} the effect of sex on salary that bypasses
  education and employment.
  \item \textbf{Indirect effect:} effects transmitted through education
  or employment.
  \item \textbf{Spurious effect:} the contribution associated with the
  sex--age association.
  \item \textbf{All effects:} direct, indirect, and spurious.
\end{itemize}
\end{tcolorbox}

\begin{tcolorbox}[colback=gray!10,colframe=gray!50,title=\textbf{User: Prompt}]
Use an empty BN set, $\{\}$: no exceptions. Remove all effects on salary.
\end{tcolorbox}

\begin{tcolorbox}[colback=blue!5,colframe=blue!30,title=\textbf{\method{}: Example Output}]
\textit{Martha is a female who enjoys spending her afternoons gardening in
her small suburban backyard. Now in the 60--74 years age bracket, she stays
active within her community. Her highest educational attainment is some
college or associate degree, which she earned decades ago. She is currently
part-time employed as a library assistant, a role that brings her great joy.
Because of her limited hours, her annual salary is less than \$25k.}
\end{tcolorbox}

%% file: sections/A8XZExperiments.tex
\section{Generation Conditioned on Protected Attributes and Covariates}
\label[appendix]{app:xz-generation}

We evaluate the setting in which the protected attribute $X$ and covariates
$\mathbf Z$ are supplied to the generator, which produces all outcome
attributes $\mathbf Y$. In the query notation of \Cref{subsec:pipeline},
this corresponds to $\mathbf C=(X,\mathbf Z)$ and $\mathbf T=\mathbf Y$.
The main-text topological experiments instead generate all attributes, with
$\mathbf C=\varnothing$ and $\mathbf T=(X,\mathbf Z,\mathbf Y)$.
We evaluate this setting on NSDUH and Census-Synth with the same four
generators as in the main experiments.
All queries are topological, and we evaluate the final outcomes: past-month
marijuana use for NSDUH and salary below \$50,000 for Census-Synth.

\subsection{Matching Reference Causal Effects}

For recovery, we sample $(X,\mathbf Z)$ jointly from rows of the original
reference data and supply these values to Vanilla, CoT, RAG, and \method{}.
We compare effects in the generated outcomes with their reference values
using the annotation and estimation protocol in
\Cref{app:effect-protocol,app:story-annotation}.

\begin{table}[t]
    \centering
    \input{tables/xz_recovery_errors_table}
    \caption{Mean absolute errors in reference-effect matching with supplied
    $(X,\mathbf Z)$, averaged across generators, in percentage points.}
    \label{tab:xz-recovery-errors}
\end{table}

\Cref{tab:xz-recovery-errors} reports mean absolute residuals, first averaging
equally over available datasets within each generator and then equally over
the four generators. The NSDUH Gemma/RAG run has a constant annotated final
outcome, causing the effect estimator to fail; its model-level mean
therefore uses Census-Synth only.
In this aggregate comparison, \method{} again achieves lower recovery error
than every baseline for all four effects: $1.59$ percentage points for TV,
$1.85$ for DE, $1.19$ for SE, and $0.73$ for IE.

\subsection{Removing Direct, Indirect, and Spurious Effects}

We evaluate \method{} with all four generators, selecting direct, indirect,
or spurious effects for removal individually, or all three together, with an
empty BN set. The supplied $(X,\mathbf Z)$ values are sampled from the
adapted distribution for each removal setting. When spurious-effect removal
is selected, we supply $(X,(\mathbf Z)_{x_0})$, where $(\mathbf Z)_{x_0}$
denotes covariates adjusted to the common reference-group distribution.
Supplying the original covariates would
retain the $X$--$\mathbf Z$ association. Direct- or indirect-effect removal
alone leaves the supplied context distribution unchanged. In each case,
we generate $\mathbf Y$ from the adapted conditional distribution given
the supplied values.
\Cref{fig:xz-effect-removal} shows that the targeted effects are reduced
toward zero in both datasets across all four generators, including the
spurious component when it is selected for removal.

\begin{figure}[t]
    \centering
    \includegraphics[width=\textwidth]{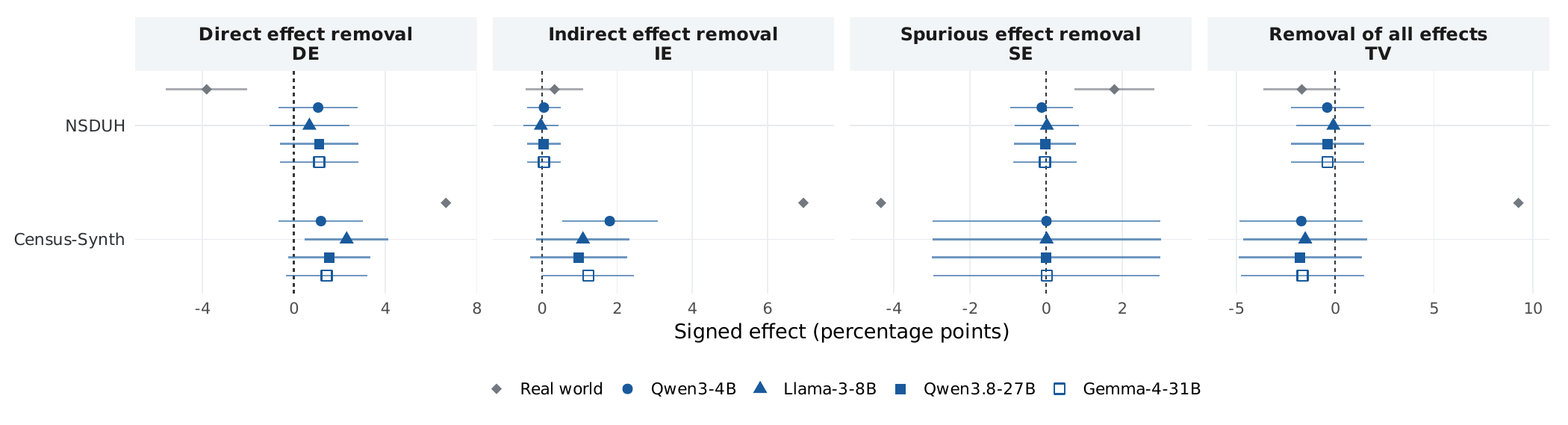}
    \caption{Effect removal for final outcomes with supplied $(X,\mathbf Z)$.}
    \label{fig:xz-effect-removal}
\end{figure}

%% file: tables/xz_recovery_errors_table.tex
\begin{tabular}{lrrrr}
\toprule
Method & TV & DE & SE & IE \\
\midrule
Vanilla & 9.18 & 7.25 & 4.92 & 3.87 \\
CoT & 9.96 & 6.85 & 5.48 & 3.82 \\
RAG & 9.20 & 5.07 & 3.43 & 9.19 \\
\method{} & \textbf{1.59} & \textbf{1.85} & \textbf{1.19} & \textbf{0.73} \\
\bottomrule
\end{tabular}

%% file: iclr2027_conference.bib
@article{bernhardt1971giggs,
  title={Griggs v. Duke Power Co.: The Implications for Private and Public Employers},
  author={Bernhardt, Herbert N},
  journal={Tex. L. Rev.},
  volume={50},
  pages={901},
  year={1971},
  publisher={HeinOnline}
}

@inproceedings{zhang2018fairness,
  title={Fairness in decision-making—the causal explanation formula},
  author={Zhang, Junzhe and Bareinboim, Elias},
  booktitle={Proceedings of the AAAI conference on artificial intelligence},
  volume={32 (1)},
  year={2018}
}

@misc{plecko2022causalfairnessanalysis,
      title={Causal Fairness Analysis}, 
      author={Drago Plecko and Elias Bareinboim},
      year={2022},
      eprint={2207.11385},
      archivePrefix={arXiv},
      primaryClass={cs.AI},
      url={https://arxiv.org/abs/2207.11385}, 
}

@book{pearl_causality,
author = {Pearl, Judea},
title = {Causality: Models, Reasoning and Inference},
year = {2009},
isbn = {052189560X},
publisher = {Cambridge University Press},
address = {USA},
edition = {2nd}
}

@inproceedings{nangia-etal-2020-crows,
    title = "{C}row{S}-Pairs: A Challenge Dataset for Measuring Social Biases in Masked Language Models",
    author = "Nangia, Nikita  and
      Vania, Clara  and
      Bhalerao, Rasika  and
      Bowman, Samuel R.",
    editor = "Webber, Bonnie  and
      Cohn, Trevor  and
      He, Yulan  and
      Liu, Yang",
    booktitle = "Proceedings of the 2020 Conference on Empirical Methods in Natural Language Processing (EMNLP)",
    month = nov,
    year = "2020",
    address = "Online",
    publisher = "Association for Computational Linguistics",
    url = "https://aclanthology.org/2020.emnlp-main.154/",
    doi = "10.18653/v1/2020.emnlp-main.154",
    pages = "1953--1967"
}

@inproceedings{nadeem-etal-2021-stereoset,
    title = "{S}tereo{S}et: Measuring stereotypical bias in pretrained language models",
    author = "Nadeem, Moin  and
      Bethke, Anna  and
      Reddy, Siva",
    editor = "Zong, Chengqing  and
      Xia, Fei  and
      Li, Wenjie  and
      Navigli, Roberto",
    booktitle = "Proceedings of the 59th Annual Meeting of the Association for Computational Linguistics and the 11th International Joint Conference on Natural Language Processing (Volume 1: Long Papers)",
    month = aug,
    year = "2021",
    address = "Online",
    publisher = "Association for Computational Linguistics",
    url = "https://aclanthology.org/2021.acl-long.416/",
    doi = "10.18653/v1/2021.acl-long.416",
    pages = "5356--5371"
}

@inproceedings{parrish-etal-2022-bbq,
    title = "{BBQ}: A hand-built bias benchmark for question answering",
    author = "Parrish, Alicia  and
      Chen, Angelica  and
      Nangia, Nikita  and
      Padmakumar, Vishakh  and
      Phang, Jason  and
      Thompson, Jana  and
      Htut, Phu Mon  and
      Bowman, Samuel R.",
    editor = "Muresan, Smaranda  and
      Nakov, Preslav  and
      Villavicencio, Aline",
    booktitle = "Findings of the Association for Computational Linguistics: ACL 2022",
    month = may,
    year = "2022",
    address = "Dublin, Ireland",
    publisher = "Association for Computational Linguistics",
    url = "https://aclanthology.org/2022.findings-acl.165/",
    doi = "10.18653/v1/2022.findings-acl.165",
    pages = "2086--2105"
}

@inproceedings{smith-etal-2022-im,
    title = "``{I}{'}m sorry to hear that'': Finding New Biases in Language Models with a Holistic Descriptor Dataset",
    author = "Smith, Eric Michael  and
      Hall, Melissa  and
      Kambadur, Melanie  and
      Presani, Eleonora  and
      Williams, Adina",
    editor = "Goldberg, Yoav  and
      Kozareva, Zornitsa  and
      Zhang, Yue",
    booktitle = "Proceedings of the 2022 Conference on Empirical Methods in Natural Language Processing",
    month = dec,
    year = "2022",
    address = "Abu Dhabi, United Arab Emirates",
    publisher = "Association for Computational Linguistics",
    url = "https://aclanthology.org/2022.emnlp-main.625/",
    doi = "10.18653/v1/2022.emnlp-main.625",
    pages = "9180--9211"
}

@misc{wang2024decodingtrustcomprehensiveassessmenttrustworthiness,
      title={DecodingTrust: A Comprehensive Assessment of Trustworthiness in GPT Models}, 
      author={Boxin Wang and Weixin Chen and Hengzhi Pei and Chulin Xie and Mintong Kang and Chenhui Zhang and Chejian Xu and Zidi Xiong and Ritik Dutta and Rylan Schaeffer and Sang T. Truong and Simran Arora and Mantas Mazeika and Dan Hendrycks and Zinan Lin and Yu Cheng and Sanmi Koyejo and Dawn Song and Bo Li},
      year={2024},
      eprint={2306.11698},
      archivePrefix={arXiv},
      primaryClass={cs.CL},
      url={https://arxiv.org/abs/2306.11698}, 
}

@misc{tamkin2023evaluatingmitigatingdiscriminationlanguage,
      title={Evaluating and Mitigating Discrimination in Language Model Decisions}, 
      author={Alex Tamkin and Amanda Askell and Liane Lovitt and Esin Durmus and Nicholas Joseph and Shauna Kravec and Karina Nguyen and Jared Kaplan and Deep Ganguli},
      year={2023},
      eprint={2312.03689},
      archivePrefix={arXiv},
      primaryClass={cs.CL},
      url={https://arxiv.org/abs/2312.03689}, 
}

@article{
bai2025explicitly,
author = {Xuechunzi Bai  and Angelina Wang  and Ilia Sucholutsky  and Thomas L. Griffiths },
title = {Explicitly unbiased large language models still form biased associations},
journal = {Proceedings of the National Academy of Sciences},
volume = {122},
number = {8},
pages = {e2416228122},
year = {2025},
doi = {10.1073/pnas.2416228122},
URL = {https://www.pnas.org/doi/abs/10.1073/pnas.2416228122},
eprint = {https://www.pnas.org/doi/pdf/10.1073/pnas.2416228122}}

@article{gallegos-etal-2024-bias,
    title = "Bias and Fairness in Large Language Models: A Survey",
    author = "Gallegos, Isabel O.  and
      Rossi, Ryan A.  and
      Barrow, Joe  and
      Tanjim, Md Mehrab  and
      Kim, Sungchul  and
      Dernoncourt, Franck  and
      Yu, Tong  and
      Zhang, Ruiyi  and
      Ahmed, Nesreen K.",
    journal = "Computational Linguistics",
    volume = "50",
    number = "3",
    month = sep,
    year = "2024",
    address = "Cambridge, MA",
    publisher = "MIT Press",
    url = "https://aclanthology.org/2024.cl-3.8/",
    doi = "10.1162/coli_a_00524",
    pages = "1097--1179"
}

@misc{kiciman2024causalreasoninglargelanguage,
      title={Causal Reasoning and Large Language Models: Opening a New Frontier for Causality}, 
      author={Emre Kıcıman and Robert Ness and Amit Sharma and Chenhao Tan},
      year={2024},
      eprint={2305.00050},
      archivePrefix={arXiv},
      primaryClass={cs.AI},
      url={https://arxiv.org/abs/2305.00050}, 
}

@article{zevcevic2023causal,
  title={Causal parrots: Large language models may talk causality but are not causal},
  author={Ze{\v{c}}evi{\'c}, Matej and Willig, Moritz and Dhami, Devendra Singh and Kersting, Kristian},
  journal={arXiv preprint arXiv:2308.13067},
  year={2023}
}

@misc{jin2024cladderassessingcausalreasoning,
      title={CLadder: Assessing Causal Reasoning in Language Models}, 
      author={Zhijing Jin and Yuen Chen and Felix Leeb and Luigi Gresele and Ojasv Kamal and Zhiheng Lyu and Kevin Blin and Fernando Gonzalez Adauto and Max Kleiman-Weiner and Mrinmaya Sachan and Bernhard Schölkopf},
      year={2024},
      eprint={2312.04350},
      archivePrefix={arXiv},
      primaryClass={cs.CL},
      url={https://arxiv.org/abs/2312.04350}, 
}

@inproceedings{
jin2024can,
title={Can Large Language Models Infer Causation from Correlation?},
author={Zhijing Jin and Jiarui Liu and Zhiheng LYU and Spencer Poff and Mrinmaya Sachan and Rada Mihalcea and Mona T. Diab and Bernhard Sch{\"o}lkopf},
booktitle={The Twelfth International Conference on Learning Representations},
year={2024},
url={https://openreview.net/forum?id=vqIH0ObdqL}
}

@inproceedings{joshi-etal-2024-llms,
    title = "{LLM}s Are Prone to Fallacies in Causal Inference",
    author = "Joshi, Nitish  and
      Saparov, Abulhair  and
      Wang, Yixin  and
      He, He",
    editor = "Al-Onaizan, Yaser  and
      Bansal, Mohit  and
      Chen, Yun-Nung",
    booktitle = "Proceedings of the 2024 Conference on Empirical Methods in Natural Language Processing",
    month = nov,
    year = "2024",
    address = "Miami, Florida, USA",
    publisher = "Association for Computational Linguistics",
    url = "https://aclanthology.org/2024.emnlp-main.590/",
    doi = "10.18653/v1/2024.emnlp-main.590",
    pages = "10553--10569"
}

@inproceedings{
abdulaal2024causal,
title={Causal Modelling Agents: Causal Graph Discovery through Synergising Metadata- and Data-driven Reasoning},
author={Ahmed Abdulaal and adamos hadjivasiliou and Nina Montana-Brown and Tiantian He and Ayodeji Ijishakin and Ivana Drobnjak and Daniel C. Castro and Daniel C. Alexander},
booktitle={The Twelfth International Conference on Learning Representations},
year={2024},
url={https://openreview.net/forum?id=pAoqRlTBtY}
}

@inproceedings{liu2024discovery,
 author = {Liu, Chenxi and Chen, Yongqiang and Liu, Tongliang and Gong, Mingming and Cheng, James and Han, Bo and Zhang, Kun},
 booktitle = {Advances in Neural Information Processing Systems},
 doi = {10.52202/079017-3249},
 editor = {A. Globerson and L. Mackey and D. Belgrave and A. Fan and U. Paquet and J. Tomczak and C. Zhang},
 pages = {102307--102365},
 publisher = {Curran Associates, Inc.},
 title = {Discovery of the Hidden World with Large Language Models},
 url = {https://proceedings.neurips.cc/paper_files/paper/2024/file/b99a07486702417d3b1bd64ec2cf74ad-Paper-Conference.pdf},
 volume = {37},
 year = {2024}
}

@inproceedings{vig2020investigating,
 author = {Vig, Jesse and Gehrmann, Sebastian and Belinkov, Yonatan and Qian, Sharon and Nevo, Daniel and Singer, Yaron and Shieber, Stuart},
 booktitle = {Advances in Neural Information Processing Systems},
 editor = {H. Larochelle and M. Ranzato and R. Hadsell and M.F. Balcan and H. Lin},
 pages = {12388--12401},
 publisher = {Curran Associates, Inc.},
 title = {Investigating Gender Bias in Language Models Using Causal Mediation Analysis},
 url = {https://proceedings.neurips.cc/paper_files/paper/2020/file/92650b2e92217715fe312e6fa7b90d82-Paper.pdf},
 volume = {33},
 year = {2020}
}

@inproceedings{zhou-etal-2023-causal,
    title = "Causal-Debias: Unifying Debiasing in Pretrained Language Models and Fine-tuning via Causal Invariant Learning",
    author = "Zhou, Fan  and
      Mao, Yuzhou  and
      Yu, Liu  and
      Yang, Yi  and
      Zhong, Ting",
    editor = "Rogers, Anna  and
      Boyd-Graber, Jordan  and
      Okazaki, Naoaki",
    booktitle = "Proceedings of the 61st Annual Meeting of the Association for Computational Linguistics (Volume 1: Long Papers)",
    month = jul,
    year = "2023",
    address = "Toronto, Canada",
    publisher = "Association for Computational Linguistics",
    url = "https://aclanthology.org/2023.acl-long.232/",
    doi = "10.18653/v1/2023.acl-long.232",
    pages = "4227--4241"
}

@inproceedings{li2025prompting,
 author = {Li, Jingling and Tang, Zeyu and Liu, Xiaoyu and Spirtes, Peter and Zhang, Kun and Leqi, Liu and Liu, Yang},
 booktitle = {International Conference on Learning Representations},
 editor = {Y. Yue and A. Garg and N. Peng and F. Sha and R. Yu},
 pages = {63606--63629},
 title = {Prompting Fairness: Integrating Causality to Debias Large Language Models},
 url = {https://proceedings.iclr.cc/paper_files/paper/2025/file/9f26a2d143a227376dff99a279f93f99-Paper-Conference.pdf},
 volume = {2025},
 year = {2025}
}

@inproceedings{kusner2017counterfactual,
 author = {Kusner, Matt J and Loftus, Joshua and Russell, Chris and Silva, Ricardo},
 booktitle = {Advances in Neural Information Processing Systems},
 editor = {I. Guyon and U. Von Luxburg and S. Bengio and H. Wallach and R. Fergus and S. Vishwanathan and R. Garnett},
 pages = {},
 publisher = {Curran Associates, Inc.},
 title = {Counterfactual Fairness},
 url = {https://proceedings.neurips.cc/paper_files/paper/2017/file/a486cd07e4ac3d270571622f4f316ec5-Paper.pdf},
 volume = {30},
 year = {2017}
}

@inproceedings{kilbertus2017avoiding,
 author = {Kilbertus, Niki and Rojas Carulla, Mateo and Parascandolo, Giambattista and Hardt, Moritz and Janzing, Dominik and Sch\"{o}lkopf, Bernhard},
 booktitle = {Advances in Neural Information Processing Systems},
 editor = {I. Guyon and U. Von Luxburg and S. Bengio and H. Wallach and R. Fergus and S. Vishwanathan and R. Garnett},
 pages = {},
 publisher = {Curran Associates, Inc.},
 title = {Avoiding Discrimination through Causal Reasoning},
 url = {https://proceedings.neurips.cc/paper_files/paper/2017/file/f5f8590cd58a54e94377e6ae2eded4d9-Paper.pdf},
 volume = {30},
 year = {2017}
}

@inproceedings{nabi2018fair,
  title={Fair inference on outcomes},
  author={Nabi, Razieh and Shpitser, Ilya},
  booktitle={Proceedings of the AAAI conference on artificial intelligence},
  volume={32 (1)},
  year={2018}
}

@inproceedings{chiappa2019path,
  title={Path-specific counterfactual fairness},
  author={Chiappa, Silvia},
  booktitle={Proceedings of the AAAI conference on artificial intelligence},
  volume={33 (01)},
  pages={7801--7808},
  year={2019}
}

@inproceedings{jung2024counterfactually,
author = {Jung, Sangwon and Yu, Sumin and Chun, Sanghyuk and Moon, Taesup},
title = {Do counterfactually fair image classifiers satisfy group fairness? - a theoretical and empirical study},
year = {2024},
isbn = {9798331314385},
publisher = {Curran Associates Inc.},
address = {Red Hook, NY, USA},
booktitle = {Proceedings of the 38th International Conference on Neural Information Processing Systems},
articleno = {1782},
numpages = {13},
location = {Vancouver, BC, Canada},
series = {NIPS '24}
}

@inproceedings{cheong2023counterfactual,
author = {Cheong, Jiaee and Kalkan, Sinan and Gunes, Hatice},
title = {Counterfactual Fairness for Facial Expression Recognition},
year = {2022},
isbn = {978-3-031-25071-2},
publisher = {Springer-Verlag},
address = {Berlin, Heidelberg},
url = {https://doi.org/10.1007/978-3-031-25072-9_16},
doi = {10.1007/978-3-031-25072-9_16},
booktitle = {Computer Vision – ECCV 2022 Workshops: Tel Aviv, Israel, October 23–27, 2022, Proceedings, Part V},
pages = {245–261},
numpages = {17},
location = {Tel Aviv, Israel}
}

@inproceedings{joo2020gender,
author = {Joo, Jungseock and K{\"a}rkk{\"a}inen, Kimmo},
title = {Gender Slopes: Counterfactual Fairness for Computer Vision Models by Attribute Manipulation},
year = {2020},
isbn = {9781450381482},
publisher = {Association for Computing Machinery},
address = {New York, NY, USA},
url = {https://doi.org/10.1145/3422841.3423533},
doi = {10.1145/3422841.3423533},
booktitle = {Proceedings of the 2nd International Workshop on Fairness, Accountability, Transparency and Ethics in Multimedia},
pages = {1–5},
numpages = {5},
location = {Seattle, WA, USA},
series = {FATE/MM '20}
}

@inproceedings{garg2019counterfactual,
author = {Garg, Sahaj and Perot, Vincent and Limtiaco, Nicole and Taly, Ankur and Chi, Ed H. and Beutel, Alex},
title = {Counterfactual Fairness in Text Classification through Robustness},
year = {2019},
isbn = {9781450363242},
publisher = {Association for Computing Machinery},
address = {New York, NY, USA},
url = {https://doi.org/10.1145/3306618.3317950},
doi = {10.1145/3306618.3317950},
booktitle = {Proceedings of the 2019 AAAI/ACM Conference on AI, Ethics, and Society},
pages = {219–226},
numpages = {8},
location = {Honolulu, HI, USA},
series = {AIES '19}
}

@misc{ons2023synthetic,
  author = {{Office for National Statistics}},
  title = {Enabling Data Access through Privacy Preserving Synthetic Data},
  year = {2023},
  month = nov,
  url = {https://datasciencecampus.ons.gov.uk/enabling-data-access-through-privacy-preserving-synthetic-data/},
  note = {Data Science Campus}
}

@misc{plecko2026causalbiasdetectiongenerative,
      title={Causal Bias Detection in Generative Artificial Intelligence}, 
      author={Drago Plecko},
      year={2026},
      eprint={2605.11365},
      archivePrefix={arXiv},
      primaryClass={cs.AI},
      url={https://arxiv.org/abs/2605.11365}, 
}

@misc{chernozhukov2018double,
  title={Double/debiased machine learning for treatment and structural parameters},
  author={Chernozhukov, Victor and Chetverikov, Denis and Demirer, Mert and Duflo, Esther and Hansen, Christian and Newey, Whitney and Robins, James},
  year={2018},
  publisher={Oxford University Press Oxford, UK}
}

@misc{samhsa2023nsduh,
  author       = {{Substance Abuse and Mental Health Services Administration}},
  title        = {{National Survey on Drug Use and Health (NSDUH)}},
  year         = {2023},
  publisher    = {U.S. Department of Health and Human Services},
  url          = {https://www.samhsa.gov/data/data-we-collect/nsduh-national-survey-drug-use-and-health},
  note         = {Accessed: 2026-09-09}
}

@misc{cdc2023brfss,
  author    = {{Centers for Disease Control and Prevention}},
  title     = {{Behavioral Risk Factor Surveillance System: 2023 Survey Data and Documentation}},
  year      = {2023},
  publisher = {U.S. Department of Health and Human Services},
  url       = {https://www.cdc.gov/brfss/annual_data/annual_2023.html},
  note      = {Accessed: 2026-09-09}
}

@misc{uscensus2023acs,
  author    = {{U.S. Census Bureau}},
  title     = {{2023 American Community Survey 1-Year Estimates}},
  year      = {2023},
  publisher = {U.S. Census Bureau},
  url       = {https://www.census.gov/data/developers/data-sets/acs-1year/2023.html},
  note      = {Accessed: 2026-09-09}
}

@Article{plecko2024fairadapt,
  title = {Fair data adaptation with quantile preservation},
  author = {Drago Plecko and Nicolai Meinshausen},
  journal = {Journal of Machine Learning Research},
  year = {2020},
  volume = {21},
  number = {242},
  pages = {1--44},
  url = {https://www.jmlr.org/papers/v21/19-966.html},
}

@article{cevid2022distributional,
  title={Distributional random forests: Heterogeneity adjustment and multivariate distributional regression},
  author={Cevid, Domagoj and Michel, Loris and N{\"a}f, Jeffrey and B{\"u}hlmann, Peter and Meinshausen, Nicolai},
  journal={Journal of Machine Learning Research},
  volume={23},
  number={333},
  pages={1--79},
  year={2022}
}

@inproceedings{madhavan-etal-2023-cfl,
    title = "{CFL}: Causally Fair Language Models Through Token-level Attribute Controlled Generation",
    author = "Madhavan, Rahul  and
      Garg, Rishabh  and
      Wadhawan, Kahini  and
      Mehta, Sameep",
    editor = "Rogers, Anna  and
      Boyd-Graber, Jordan  and
      Okazaki, Naoaki",
    booktitle = "Findings of the Association for Computational Linguistics: ACL 2023",
    month = jul,
    year = "2023",
    address = "Toronto, Canada",
    publisher = "Association for Computational Linguistics",
    url = "https://aclanthology.org/2023.findings-acl.720/",
    doi = "10.18653/v1/2023.findings-acl.720",
    pages = "11344--11358"
}

@inproceedings{brown2020language,
 author = {Brown, Tom and Mann, Benjamin and Ryder, Nick and Subbiah, Melanie and Kaplan, Jared D and Dhariwal, Prafulla and Neelakantan, Arvind and Shyam, Pranav and Sastry, Girish and Askell, Amanda and Agarwal, Sandhini and Herbert-Voss, Ariel and Krueger, Gretchen and Henighan, Tom and Child, Rewon and Ramesh, Aditya and Ziegler, Daniel and Wu, Jeffrey and Winter, Clemens and Hesse, Chris and Chen, Mark and Sigler, Eric and Litwin, Mateusz and Gray, Scott and Chess, Benjamin and Clark, Jack and Berner, Christopher and McCandlish, Sam and Radford, Alec and Sutskever, Ilya and Amodei, Dario},
 booktitle = {Advances in Neural Information Processing Systems},
 editor = {H. Larochelle and M. Ranzato and R. Hadsell and M.F. Balcan and H. Lin},
 pages = {1877--1901},
 publisher = {Curran Associates, Inc.},
 title = {Language Models are Few-Shot Learners},
 url = {https://proceedings.neurips.cc/paper_files/paper/2020/file/1457c0d6bfcb4967418bfb8ac142f64a-Paper.pdf},
 volume = {33},
 year = {2020}
}

@article{kasneci2023chatgpt,
title = {ChatGPT for good? On opportunities and challenges of large language models for education},
journal = {Learning and Individual Differences},
volume = {103},
pages = {102274},
year = {2023},
issn = {1041-6080},
doi = {https://doi.org/10.1016/j.lindif.2023.102274},
url = {https://www.sciencedirect.com/science/article/pii/S1041608023000195},
author = {Enkelejda Kasneci and Kathrin Sessler and Stefan Küchemann and Maria Bannert and Daryna Dementieva and Frank Fischer and Urs Gasser and Georg Groh and Stephan Günnemann and Eyke Hüllermeier and Stephan Krusche and Gitta Kutyniok and Tilman Michaeli and Claudia Nerdel and Jürgen Pfeffer and Oleksandra Poquet and Michael Sailer and Albrecht Schmidt and Tina Seidel and Matthias Stadler and Jochen Weller and Jochen Kuhn and Gjergji Kasneci}
}

@article{singhal2023large,
  title={Large language models encode clinical knowledge},
  author={Singhal, Karan and Azizi, Shekoofeh and Tu, Tao and Mahdavi, S Sara and Wei, Jason and Chung, Hyung Won and Scales, Nathan and Tanwani, Ajay and Cole-Lewis, Heather and Pfohl, Stephen and others},
  journal={Nature},
  volume={620},
  number={7972},
  pages={172--180},
  year={2023},
  publisher={Nature Publishing Group UK London}
}

@inproceedings{an-etal-2024-large,
    title = "Do Large Language Models Discriminate in Hiring Decisions on the Basis of Race, Ethnicity, and Gender?",
    author = "An, Haozhe  and
      Acquaye, Christabel  and
      Wang, Colin  and
      Li, Zongxia  and
      Rudinger, Rachel",
    editor = "Ku, Lun-Wei  and
      Martins, Andre  and
      Srikumar, Vivek",
    booktitle = "Proceedings of the 62nd Annual Meeting of the Association for Computational Linguistics (Volume 2: Short Papers)",
    month = aug,
    year = "2024",
    address = "Bangkok, Thailand",
    publisher = "Association for Computational Linguistics",
    url = "https://aclanthology.org/2024.acl-short.37/",
    doi = "10.18653/v1/2024.acl-short.37",
    pages = "386--397"
}

@misc{yang2025qwen3,
      title={Qwen3 Technical Report}, 
      author={An Yang and Anfeng Li and Baosong Yang and Beichen Zhang and Binyuan Hui and Bo Zheng and Bowen Yu and Chang Gao and Chengen Huang and Chenxu Lv and Chujie Zheng and Dayiheng Liu and Fan Zhou and Fei Huang and Feng Hu and Hao Ge and Haoran Wei and Huan Lin and Jialong Tang and Jian Yang and Jianhong Tu and Jianwei Zhang and Jianxin Yang and Jiaxi Yang and Jing Zhou and Jingren Zhou and Junyang Lin and Kai Dang and Keqin Bao and Kexin Yang and Le Yu and Lianghao Deng and Mei Li and Mingfeng Xue and Mingze Li and Pei Zhang and Peng Wang and Qin Zhu and Rui Men and Ruize Gao and Shixuan Liu and Shuang Luo and Tianhao Li and Tianyi Tang and Wenbiao Yin and Xingzhang Ren and Xinyu Wang and Xinyu Zhang and Xuancheng Ren and Yang Fan and Yang Su and Yichang Zhang and Yinger Zhang and Yu Wan and Yuqiong Liu and Zekun Wang and Zeyu Cui and Zhenru Zhang and Zhipeng Zhou and Zihan Qiu},
      year={2025},
      eprint={2505.09388},
      archivePrefix={arXiv},
      primaryClass={cs.CL},
      url={https://arxiv.org/abs/2505.09388}, 
}

@article{grattafiori2024llama,
  title={The llama 3 herd of models},
  author={Grattafiori, Aaron and Dubey, Abhimanyu and Jauhri, Abhinav and Pandey, Abhinav and Kadian, Abhishek and Al-Dahle, Ahmad and Letman, Aiesha and Mathur, Akhil and Schelten, Alan and Vaughan, Alex and others},
  journal={arXiv preprint arXiv:2407.21783},
  year={2024}
}

@misc{qwen38,
    title = {{Qwen3.8-Max}: A New Bar for Coding and Cowork},
    url = {https://qwen.ai/blog?id=qwen3.8},
    author = {{Qwen Team}},
    month = {August},
    year = {2026}
}

@article{gemmateam2026gemma4,
  title={Gemma 4 technical report},
  author={{Gemma Team} and Abd, Sherif El and Aggarwal, Vaibhav and Algayres, Robin and Andreev, Alek and Bachem, Olivier and Ballantyne, Ian and Brick, Cormac and C{\u{a}}rbune, Victor and Casbon, Michelle and others},
  journal={arXiv preprint arXiv:2607.02770},
  year={2026}
}

@article{gneiting2007strictly,
  title={Strictly proper scoring rules, prediction, and estimation},
  author={Gneiting, Tilmann and Raftery, Adrian E},
  journal={Journal of the American statistical Association},
  volume={102},
  number={477},
  pages={359--378},
  year={2007},
  publisher={Taylor \& Francis}
}
